%% file: root.tex
\documentclass[letterpaper, 10 pt, conference]{ieeeconf}  

\IEEEoverridecommandlockouts                              

\input{packages.tex}
\usepackage[backend=bibtex,style=ieee,mincitenames=1,maxcitenames=2,maxbibnames=20,url=false,doi=false]{biblatex}
\title{Constrained Hierarchical Monte Carlo Belief-State Planning
}
\author {
    Arec Jamgochian,$^1$ Hugo Buurmeijer,$^1$ Kyle H. Wray,$^1$ Anthony Corso,$^1$ Mykel J. Kochenderfer$^1$ 
\thanks{$^{1}$Stanford University, Stanford, CA 94305 USA
{\tt\small \{arec, hbuurmei, kylewray, acorso, mykel\}@stanford.edu}}%
\thanks{$^\dagger$This material is based upon work supported by the National Science Foundation Graduate Research Fellowship Program under Grant No. DGE-1656518. 
This work is also supported by the COMET K2---Competence Centers for Excellent Technologies Programme of the Federal Ministry for Transport, Innovation and Technology (bmvit), the Federal Ministry for Digital, Business and Enterprise (bmdw), the Austrian Research Promotion Agency (FFG), the Province of Styria, and the Styrian Business Promotion Agency (SFG). 
}
}

\begin{document}

\maketitle

\begin{abstract}
\input{sections/00-abstract.tex}
\end{abstract}

\input{sections/01-introduction.tex}

\input{sections/02-background.tex}
\input{sections/03-method.tex}
\input{sections/04-experiments.tex}

\input{sections/05-conclusion.tex}

\renewcommand*{\bibfont}{\small}
\printbibliography
\end{document}

%% file: packages.tex
\usepackage{multirow}
\usepackage{amsmath,amssymb,amsfonts}
\usepackage{xcolor}
\usepackage{siunitx}
\usepackage{booktabs}
\usepackage{subcaption}

\usepackage{flushend}

\newtheorem{definition}{Definition}

\usepackage{pgfplots}
\pgfplotsset{compat=newest}
\usepgfplotslibrary{groupplots}
\usepgfplotslibrary{dateplot}
\usetikzlibrary{arrows.meta}
\usetikzlibrary{backgrounds}
\usepgfplotslibrary{patchplots}
\usepgfplotslibrary{fillbetween}
\pgfplotsset{%
     layers/standard/.define layer set={%
         background,axis background,axis grid,axis ticks,axis lines,axis tick labels,pre main,main,axis descriptions,axis foreground%
     }{
         grid style={/pgfplots/on layer=axis grid},%
         tick style={/pgfplots/on layer=axis ticks},%
         axis line style={/pgfplots/on layer=axis lines},%
         label style={/pgfplots/on layer=axis descriptions},%
         legend style={/pgfplots/on layer=axis descriptions},%
         title style={/pgfplots/on layer=axis descriptions},%
         colorbar style={/pgfplots/on layer=axis descriptions},%
         ticklabel style={/pgfplots/on layer=axis tick labels},%
         axis background@ style={/pgfplots/on layer=axis background},%
         3d box foreground style={/pgfplots/on layer=axis foreground},%
     },
 }

\usepackage{graphicx}
\usepackage{url}
\usepackage[capitalize]{cleveref}
\usepackage{csquotes}
\usepackage{vector}

\usepackage{algorithm}
\usepackage[noend]{algpseudocode}

\newcommand{\ra}[1]{\renewcommand{\arraystretch}{#1}}
\newcolumntype{R}{>{$}r<{$}}
\input{commands}

%% file: commands.tex
\newcommand{\algname}{COBeTS}
\newcommand{\ep}{\ensuremath{e}}
\newcommand{\argmax}{\operatornamewithlimits{arg\,max}}

\newcommand{\blue}[1]{\textcolor{blue}{#1}}

\newtheorem{proposition}{Proposition}

%% file: sections/00-abstract.tex
Optimal plans in Constrained Partially Observable Markov Decision Processes (CPOMDPs) maximize reward objectives while satisfying hard cost constraints, generalizing safe planning under state and transition uncertainty. Unfortunately, online CPOMDP planning is extremely difficult in large or continuous problem domains. In many large robotic domains, hierarchical decomposition can simplify planning by using tools for low-level control given high-level action primitives (options). 
We introduce Constrained Options Belief Tree Search (\algname) to leverage this hierarchy and scale online search-based CPOMDP planning to large robotic problems. 
We show that if primitive option controllers are defined to satisfy assigned constraint budgets, then \algname\ will satisfy constraints anytime. Otherwise, \algname\ will guide the search towards a safe sequence of option primitives, and hierarchical monitoring can be used to achieve runtime safety. We demonstrate \algname\ in several safety-critical, constrained partially observable robotic domains, showing that it can plan successfully in continuous CPOMDPs while non-hierarchical baselines cannot. 

%% file: sections/01-introduction.tex
\section{Introduction}

Planning in robotics requires robust regard for safety, which often necessitates careful consideration of uncertainty. Two factors contributing to uncertainty include a) the true state of the robot and surrounding environment (\textit{state} uncertainty), and b) how that state will evolve given robot actuation (\textit{transition} uncertainty). Constrained partially observable Markov decision processes (CPOMDPs) provide a general mathematical framework for safe planning under state and transition uncertainty by imposing constraints~\cite{poupart2015approximate}. 

While offline CPOMDP planning algorithms are able to build policies for discrete environments with thousands of possible states~\cite{poupart2015approximate}, building policies in many robotic domains that are typically large or continuous necessitates online planning. Online CPOMDP planning has been scaled to large or continuous state spaces by combining Monte Carlo tree search with Lagrangian exploration and dual ascent~\cite{lee2018monte} and has recently been extended to domains with continuous action and observation spaces~\cite{jamgochian2023online}. However, search-based planning with large or continuous action and observation spaces is still extremely difficult. Common techniques try to artificially limit the width of the search tree by restricting the sets of successor nodes~\cite{couetoux2011continuous,sunberg2018online}. Unfortunately, it can often be difficult to guarantee the inclusion of promising actions in the reduced search set. This problem accelerates as searches deepen, which can be especially problematic when problems require deep searches to find promising action sequences. Models can be learned for biasing action selection~\cite{mern2021bayesian,cai2022closing,moss2023betazero}, however, this requires generating data from past experience.

\begin{figure}[t]
\centering
     \begin{subfigure}[b]{0.230\textwidth}
         \centering
         \includegraphics[width=1.01\textwidth]{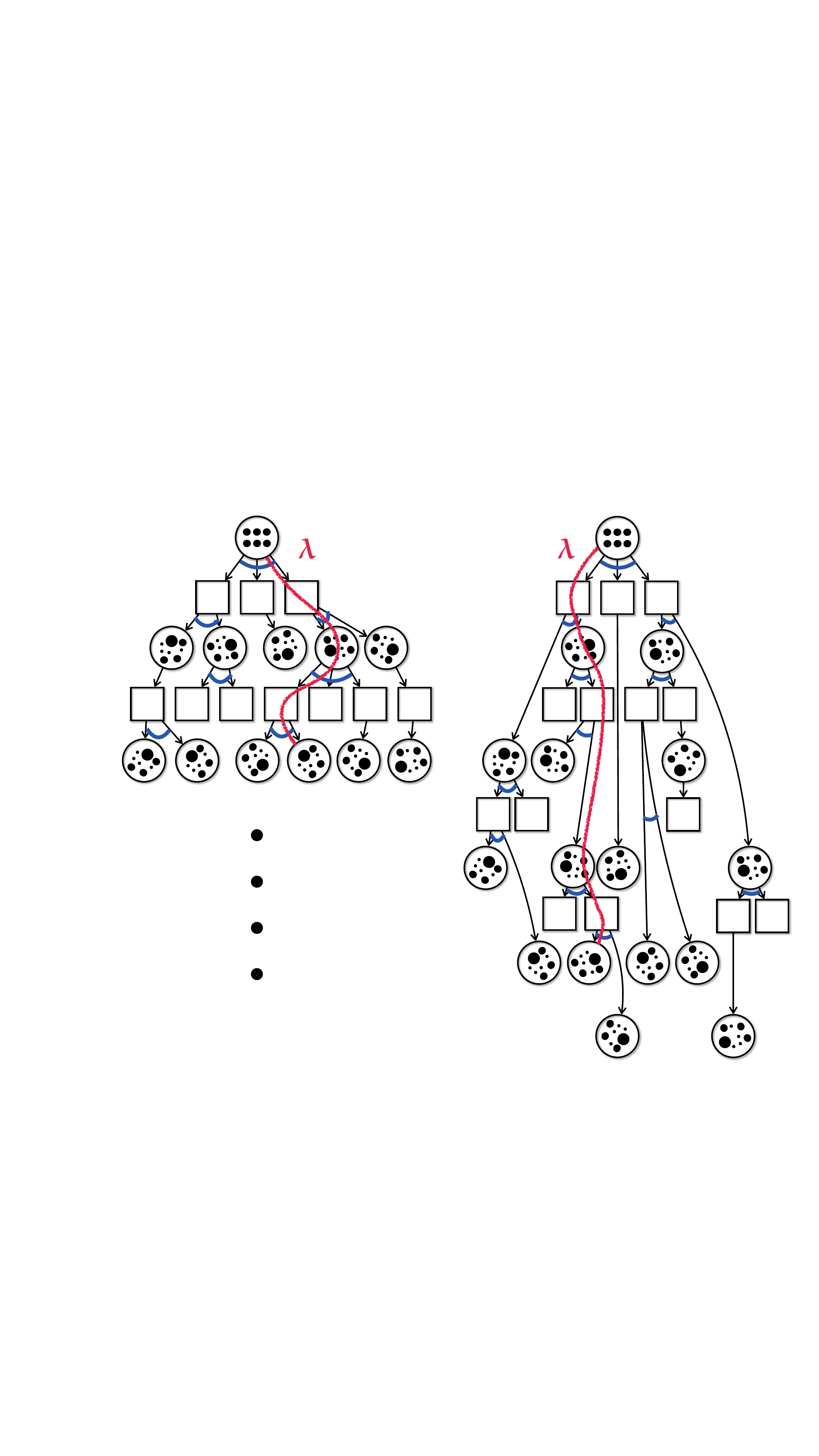}
         \caption{CPFT-DPW Tree~\cite{jamgochian2023online}}
     \end{subfigure}
     \hfill
     \begin{subfigure}[b]{0.233\textwidth}
         \centering
         \includegraphics[width=1.01\textwidth]{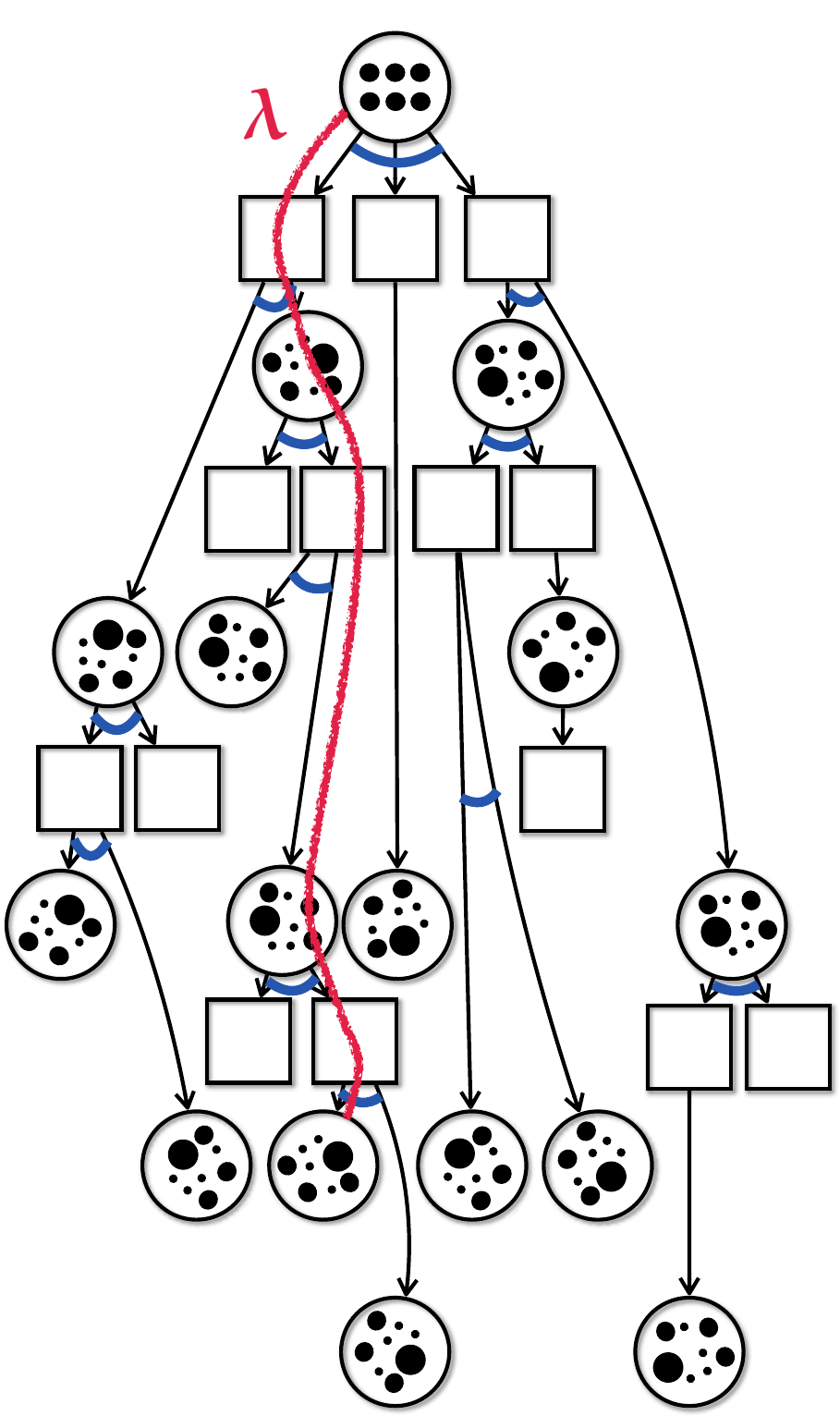}
         \caption{\algname\ Tree}
     \end{subfigure}
\caption{In CPFT-DPW~\cite{jamgochian2023online} (left), progressive widening (blue) is used to limit the branching factor of the Monte Carlo belief-state search tree, while dual parameters (red) are optimized to guide the search towards constraint satisfaction. \algname\ (right), leverages a hierarchy to decompose the partially observable planning problem, resulting in a search tree over options and semi-Markov belief transitions, with potentially far fewer nodes.} 
\label{fig1}
\vspace{-0.2cm}
\end{figure}

In many robotic planning applications, low-level controllers can be easily crafted for different high-level action primitives using domain expertise or commonly available tools (e.g. trajectory optimization). Decomposing search hierarchically over these action primitives (macro-actions, options) can \textit{significantly} reduce the size of the search tree. Decomposition can reduce the space of actions to search over, but more importantly, reduces the search depth required to plan to the same horizon.

In this paper, we introduce Constrained Options Belief Tree Search (\algname), a Monte Carlo tree search algorithm that leverages hierarchical decomposition to scale online CPOMDP planning. \algname, depicted in~\Cref{fig1}, combines the options framework to handle hierarchies~\cite{sutton1999between}, particle filter tree search to search over beliefs~\cite{sunberg2018online}, 
progressive widening to limit the number of observation \textit{sequences} emitted from each option node, and Lagrangian exploration with dual ascent to guide the search towards safe primitives~\cite{lee2018monte}. We show that if options can satisfy assigned constraint budgets, \algname\ satisfies constraints anytime. If not, dual ascent will guide the search toward safety, satisfying constraints in the limit. In our experiments, we demonstrate \algname\ on a toy domain, a carbon sequestration problem, and two robot localization problems. In each of these problems, \algname\ significantly outperforms state-of-the-art baselines that plan without hierarchical decomposition. Additionally, we demonstrate that \algname\ can satisfy constraints anytime with feasible options, and \algname\ with many options can outperform baselines because of the overall reduction in tree size induced by a hierarchy. To our knowledge, this is the first work explicitly formulating hierarchical CPOMDP planning.

In summary, our contributions are to:
\begin{itemize}
\item introduce \algname\ to perform online CPOMDP planning in large or continuous domains by using hierarchical decomposition,
\item examine its anytime safety properties and tree complexity reduction, and
\item demonstrate \algname\ on four constrained partially observable problems, including two robotic problems where non-hierarchical baselines fail to plan successfully.
\end{itemize}

%% file: sections/02-background.tex
\section{Background}
\subsection{CPOMDPs}

A CPOMDP is defined by the tuple $(\mathcal{S}, \mathcal{A}, \mathcal{O}, T, Z, R, \vect{C}, \hat{\vect{c}}, \gamma)$ consisting of state, action, and observation spaces $\mathcal{S}, \mathcal{A}, \mathcal{O}$, a transition model $T$ mapping states and actions to a distribution over resulting states, an observation model $Z$ mapping an underlying transition to a distribution over emitted observations, a reward function $R$ and cost function $\vect{C}$ mapping an underlying state transition to an instantaneous reward and vector of instantaneous, non-negative costs, a vector of cost budgets $\hat{\vect{c}}$, and a discount factor $\gamma$. A policy $\pi$ generates actions from an initial state distribution $b_0$ and a history of actions $a_{0:t}$ and observations $o_{1:t}$, which together can be represented concisely as an instantaneous belief distribution over states $b_t$ where $b_t(s)=p(s_t=s\mid b_0, a_{0:t}, o_{1:t})$. An optimal policy acts to maximize expected discounted reward while satisfying expected discounted cost budgets, that is, to optimize the following: 
\begin{align}
\max_\pi & \ V_R^\pi(b_0)= \mathbb{E}_\pi \left[\sum_{t=0}^\infty \gamma^t R(b_t,a_t) \mid b_0 \right]  \label{eq:cpomdp-objective-reward} \\
\text{s.t.   } & V_{C_k}^\pi(b_0)= \mathbb{E}_\pi \left[\sum_{t=0}^\infty \gamma^t C_k(b_t,a_t) \mid b_0 \right] \leq \hat{c}_k \ \forall \ k\text{,} \label{eq:cpomdp-objective-constraints}
\end{align}
where belief-based reward and cost functions return the expected reward and costs from transitions from states in those beliefs~\cite{sondik1978optimal,kochenderfer2022algorithms,altman1999constrained,poupart2015approximate,isom2008piecewise}. 

Offline CPOMDP planning algorithms solve for compact policies that map from any history to instantaneous actions. Offline CPOMDP solution methods include dynamic programming~\cite{isom2008piecewise,kim2011point}, approximate linear programming~\cite{poupart2015approximate}, column generation~\cite{walraven2018column}, and projected gradient ascent~\cite{wray2022scalable}. Unfortunately, offline solutions are limited to problems with small state, action, and observation spaces. 

Online CPOMDP planning algorithms can generate better solutions by searching across immediately reachable beliefs~\cite{undurti2010online}. CC-POMCP~\cite{lee2018monte} plans online in extremely large state spaces by performing Lagrangian-guided partially observable Monte Carlo tree search~\cite{silver2010monte}, with dual ascent to optimize the Lagrange multipliers. 
CPFT-DPW and CPOMCPOW~\cite{jamgochian2023online} scale constrained online Monte Carlo planning to large action and observation spaces by limiting the search tree branching factor using progressive widening~\cite{couetoux2011continuous,sunberg2018online}.

\subsection{Hierarchical Planning}

Hierarchical planning simplifies difficult planning problems by favorably decomposing them into more tractable subproblems~\cite{sacerdoti,parr1997reinforcement}.
Applications to robotics date back to the use of high-level task planning and low-level execution and runtime monitoring on the Shakey robot~\cite{fikes1972learning}. One common framework for planning hierarchically is through options \cite{sutton1999between}, in which a primitive controller (option) ${\hat{a}}$ is chosen from a finite set of options $\hat{\mathcal{A}}$ and executed until termination, upon which a new valid option is chosen. A partially observable options model augments an underlying POMDP problem with the set $\{\mathcal{I}_{\hat{a}}, \pi^L_{\hat{a}}, \beta_{\hat{a}}\}_{{\hat{a}} \in \hat{\mathcal{A}}}$, that for every option, defines a set of belief-states $\mathcal{I}$ from which it can be initialized, a low-level control policy $\pi^L(a \mid b)$ returning actions from the underlying action space $\mathcal{A}$, and a function $\beta$ that returns the probability that an option will terminate in a given belief. 
When beginning execution of a new option, the high-level policy $\pi$ must choose from the subset of options that are available from the instantaneous belief, $\{\hat{a} \mid b_t \in  \mathcal{I}_{\hat{a}}\}$. If desired, the underlying action space can be included in the set of options.

An implementation of hierarchical policy execution in CPOMDPs using the options framework is depicted in \Cref{alg:common}. During execution, the low-level policy acts in the environment (lines 6--7), updates the cost budget with the expected instantaneous cost (line 8), and updates the state belief using a new observation (line 9). A high-level \texttt{SelectOption} policy chooses a new option whenever an executing option terminates (lines 4--5).
\input{algs/common.tex}

Related previous work performed online hierarchical planning for \textit{unconstrained} POMDPs~\cite{vien2015hierarchical} by combining partially observable MCTS with MaxQ~\cite{dietterich1998maxq}, an alternative framework for hierarchical planning. Though not explicitly using hard constraints, additional recent work has focused on safe planning in large partially observable robotic domains by using a hierarchical information roadmap to manage local risks safely on large exploration missions~\cite{kim2021plgrim}. 
For problems with favorable state partitions (e.g. path planning on a grid of neighboring states), work has also been done to solve large, \textit{fully-observable}, constrained MDPs hierarchically by combining global CMDP solutions over coarse partitions with local solutions over underlying states~\cite{feyzabadi2017planning}. 

%% file: algs/common.tex
\begin{algorithm}[t]
    \caption{Hierarchical Execution in an Options CPOMDP} \label{alg:common}
    \begin{algorithmic}[1]
        \Procedure{Execute}{$b_0,\hat{\vect{c}}$}
            \State $\hat{a} \gets \emptyset, b \gets b_0$
            \While $\neg \Call{Terminal}{b}$
                \If $\hat{a} = \emptyset \lor \Call{Terminate}{\hat{a},b}$
                \State $\hat{a} \gets \Call{SelectOption}{b,\hat{\vect{c}}}$
                \EndIf
                \State $a \gets \Call{Action}{\hat{a},b}$
                \State $o \gets \Call{Step}{b, a}$
                \State $\hat{\vect{c}} \gets \left[\frac{\hat{\vect{c}}-\vect{C}(b,a)}{\gamma}\right]^+$
                \State $b \gets \Call{UpdateBelief}{b,a,o}$
            \EndWhile
        \EndProcedure
    \end{algorithmic}
\end{algorithm}

%% file: sections/03-method.tex
\section{Methodology}

\subsection{Preliminaries}

We formulate hierarchical planning in a CPOMDP using the options framework. Options induce two processes: a low-level Markov process over the underlying state, action, and observation space, and a high-level semi-Markov process between successive option calls.
With option policies defined a priori, the underlying model is a constrained partially observable semi-Markov decision process (CPOSMDP).
Consequently, we now briefly cover the CSMDP, its generalization as CPOSMDP, and the CPOSMDP's equivalent belief-state CSMDP.
More details can be found in related work such as by~\citeauthor{vien2015hierarchical}~\cite{vien2015hierarchical}.

A CSMDP is defined in a way similar to a CMDP, with the inclusion of a now semi-Markov transition function $T$ that defines the joint probability of the successor state alongside the number of steps required to transition given a state and action, $p(s', \tau \mid s, a)$.
Decisions are made at successive decision epochs $\ep$, each indexing a time step $t_\ep$ when an action $a_\ep$ begins executing and its duration $\tau_\ep$ where $t_{\ep + 1} = t_{\ep} + \tau_{\ep}$.

Similarly, a CPOSMDP is defined with the tuple $(\mathcal{S}, \mathcal{A}, \mathcal{O}, P, R, \vect{C}, \hat{\vect{c}}, \gamma)$, where $P$ models the joint semi-Markov transition and observation functions $p(s', \vect{o}, \tau \mid s, a)$, where $\vect{o} \in \mathcal{O}^\tau$ is the sequence of emitted observations in a $\tau$-step semi-Markov transition.
As with POMDPs and their equivalent belief-state MDP representations, we can define a CPOSMDP equivalently as a belief-state CSMDP $(\tilde{\mathcal{S}}, \mathcal{A}, \tilde{T}, \tilde{R}, \tilde{\vect{C}}, \hat{\vect{c}}, \gamma)$.
The belief states are $b \in \tilde{\mathcal{S}} = \triangle(\mathcal{S})$, with transitions $\tilde{T}(b, a, b'=ba\vect{o}, \tau) = \sum_{s, s' \in \mathcal{S}}b(s)b'(s')P(s, a, s', \vect{o}, \tau)$, rewards $\tilde{R}(b, a) = \sum_{s \in \mathcal{S}}b(s)R(s, a)$, and costs $\tilde{\vect{C}}(b, a) = \sum_{s \in \mathcal{S}}b(s)\vect{C}(s, a)$. 

\begin{proposition}
    \label{prop:belief-state-CSMDP-equivalent-to-CPOSMDP}
    For all policies $\pi$, the reward value functions and cost value functions of the belief-state CSMDP are equal to those of the CPOSMDP, that is, for all $b\in\tilde{\mathcal{S}}, \tilde{V}^\pi_R(b) = V^\pi_R(b)$ and $\tilde{\vect{V}}^\pi_{C}(b)=\vect{V}^\pi_{C}(b)$.
\end{proposition}

The proof follows directly from Theorem 2 of ~\citeauthor{vien2015hierarchical}~\cite{vien2015hierarchical}, with vector costs and cost-values treated analogously as scalar rewards and reward-values in the original proof. This theoretical result lays the groundwork for \algname\ as it allows us to solve CPOSMDPs by solving their equivalent belief-state CSMDPs.

\input{algs/cobt.tex}

\subsection{Constrained Options Belief-Tree Search (\algname)}

The idea behind \algname\ is to select new options in a large CPOSMDP by planning over the equivalent belief-state CSMDP. 
\algname\ augments CPFT-DPW~\cite{jamgochian2023online}, a recent algorithm for online belief-state CMDP planning, with careful consideration for the options framework.
That is, rather than search over actions $a$, \algname\ searches over options $\hat{a}$ and samples their induced semi-Markov belief-state transitions. 

The \algname\ option-selection procedure is outlined in~\Cref{alg:cobets} with changes from CPFT-DPW highlighted in blue. The procedure is a recursive Monte Carlo tree search on a particle filter belief-state $b$. In \texttt{OptionProgWiden}, option selection (lines 13--14) is guided by a Lagrangian upper confidence bound heuristic that uses the current estimate of the dual parameters to trade off between reward and constraint objectives (line 4) and an exploration bonus based on visit counts. Dual parameters are updated between searches through dual ascent (line 7), in which constraint violations during search induce strengthening of the dual parameters and vice versa. 

After selecting a search option (line 18), \algname\ imagines executing that option until its termination, resulting in a semi-Markov belief transition (lines 20--27). Option execution uses sampled low-level actions and observations to update the $m$-state particle filter belief at every step (line 23) while tracking the discounted rewards and costs accumulated along the option trajectory. 
New leaf nodes are initialized with value estimates that can be generated from default policy rollouts or heuristics (\texttt{EstimateValue} in line 28). 
The simulated reward and cost values are used to make temporal difference updates and are then backpropagated (lines 32--38). 

Searching across a large set of options or resulting transitions necessitates techniques to limit the size of the tree. Progressive widening artificially limits the branching factor of a node as a function of its visit count $N(b)$, limiting the number of children to $|Ch(b)| \approx kN(b)^\alpha$, where $k>0$ and $\alpha\in (0,1)$ are hyperparameters that control the shape of the widening~\cite{couetoux2011continuous,sunberg2018online}. \algname\ implements progressive widening on the option space (lines 9--14) and on the semi-Markov belief-state transition space (line 19). 
Different option sampling strategies in \texttt{SampleNextOption} can ensure coverage of the option space~\cite{lim2021voronoi}. 
Since the transition distribution $T = p(b', \tau \mid b, \hat{a})$ is often continuous and uncountable, COBeTS benefits greatly from progressive widening on its belief transitions for the same reasons as large continuous POMDPs~\cite{sunberg2018online}.

Hierarchical decomposition provides computational advantages that can be estimated through the ratio in sizes between analogous CMDP and CSMDP search trees.
Consider searching across a belief-state CMDP with average action branching of cardinality $A$ and state transition branching with average cardinality $O$ alongside a belief-state CSMDP with action branching of average cardinality $c_1A$, and state transition branching with average cardinality $c_2O$ after an average of $\tau$ steps.     

With a fixed time horizon $T$, searching over the CSMDP instead of the CMDP improves computational complexity by a factor of $\mathcal{O}((c_1c_2)^{-T/\tau}(AO)^{T(\tau-1)/\tau})$, as the ratio in the computational complexity can be expressed as the ratio of the tree sizes:
\begin{equation} \label{eq:treesize}
\frac{\text{CSMDP size}}{\text{CMDP size}} = \mathcal{O}\left(\frac{(c_1c_2AO)^\frac{T}{\tau}}{(AO)^T}\right)= \mathcal{O}\left(\frac{(c_1c_2)^\frac{T}{\tau}}{(AO)^{\frac{T(\tau-1)}{\tau}}}\right)\text{.}
\end{equation}
This ratio allows us to analyze the significant improvement in computational complexity induced through a hierarchical decomposition.
If our hierarchical decomposition had the same action and observation branching factors ($c_1 = c_2 = 1$), it would result in a tree that is smaller by a factor of $(AO)^{\frac{T(\tau-1)}{\tau}}$.
This gives significant leeway and allows us to, for example, compensate for designing a large number of many-step options. 

\subsection{Maintaining Feasibility Anytime with Options}

Though combining Lagrangian Monte Carlo Tree Search with dual ascent guides search away from constraint violations in the limit, it does not guarantee anytime constraint satisfaction. When executing a hierarchical controller, runtime monitoring can be used to ensure safety online by terminating options and replanning in case of impending constraint violations. However, unsafe search could still lead to states where safe replanning is not possible. In this section, we show that when options are feasible when executed under an assigned budget for their decision epoch, \algname\ can maintain global feasibility anytime. To show this, we first define local feasibility, one-step global feasibility, and global feasibility.

\begin{definition}
    \label{def:locally-feasible-option}
    An option $\hat{a}_\ep$ chosen at decision epoch $e$ in $b_\ep$ is \emph{locally feasible} given a budget $\hat{\mathbf{c}}_\ep$ if $\mathbf{Q}_\mathbf{C}^\pi(b_\ep, \hat{a}_\ep) \leq \hat{\mathbf{c}}_e$.
\end{definition}
\begin{definition}
    \label{def:one-step-global-feasibility}
    An option $\hat{a}_\ep$ chosen at decision epoch $\ep$ in $b_\ep$ is \emph{one-step globally feasible} if $\mathbf{C}_{\ep-1} + \mathbf{Q}_{\mathbf{C},\ep}^\pi(b_\ep, \hat{a}_\ep) \leq \hat{\mathbf{c}}$, where $\mathbf{C}_{\ep-1} = \sum_{\ep'=0}^{\ep-1} \gamma^{t_{\ep'}}\tilde{\mathbf{c}}_{\ep'} = \sum_{t=0}^{t_\ep-1} \gamma^{t}\mathbf{C}(b_t,a_t)$.
\end{definition}
\begin{definition}
    \label{def:global-feasibility}
    An algorithm or policy $\pi$ is said to be \emph{globally feasible} if $\mathbf{V}_{\mathbf{C}}^\pi(b_0) \leq \hat{\mathbf{c}}$.
\end{definition}

Informally, a locally feasible option is guaranteed to satisfy a set cost budget while it is in control, a one-step globally feasible option can be applied once and is guaranteed to satisfy the global budget, and global feasibility states that the original constraints from~\cref{eq:cpomdp-objective-constraints} are satisfied.

With these definitions,~\cref{prop:local-feasible-option-with-hatct-one-step-globally-feasible} below shows that for a CPOSMDP, if a locally feasible option is chosen with a particular assignment of $\hat{\mathbf{c}}_\ep$, then it ensures that one-step is globally feasible.

\begin{proposition}
    \label{prop:local-feasible-option-with-hatct-one-step-globally-feasible}
    For policy $\pi$ and locally feasible option $\hat{a}_\ep$ at decision epoch $\ep$ with accumulated costs $\vect{C}_{\ep-1}$, if $\hat{\mathbf{c}}_\ep$ = $\frac{\hat{\mathbf{c}} - \mathbf{C}_{\ep-1}}{\gamma^{t_{\ep}}} \geq 0$ then $\hat{a}_\ep$ is one-step globally feasible.
\end{proposition}

\begin{proof}
By definition of a locally feasible option and the choice of $\hat{\mathbf{c}}_e$:
\begin{align*}
    &\mathbf{Q}_{\mathbf{C}}(b_\ep, \hat{a}_\ep) \leq \hat{\mathbf{c}}_\ep = \frac{\hat{\mathbf{c}} - \mathbf{C}_{\ep-1}}{\gamma^{t_{\ep}}}\\
    &\mathbf{C}_{\ep-1} + \gamma^{t_{\ep}} \mathbb{E}\Big[\sum_{t=t_{\ep}}^\infty \gamma^{t-t_{\ep}} \mathbf{C}(b_t, a_t) \mid b_\ep, \hat{a}_\ep,\pi \Big] \leq \hat{\mathbf{c}} \\
    &\mathbf{C}_{\ep-1} + \mathbf{Q}_{\mathbf{C},\ep}^\pi(b_\ep, \hat{a}_\ep) \leq \hat{\mathbf{c}} \text{.}
\end{align*}
Thus, by~\cref{def:one-step-global-feasibility}, $\hat{a}_e$ is one-step globally feasible.
\end{proof}

These results imply that \algname\ is globally feasible when its options are locally feasible, that is, when they satisfy the budgets passed to \texttt{SampleNextOption} (line 11).

\begin{proposition}
    \label{prop:cobets-globally-feasible}
    \algname\ is globally feasible if all its options $\hat{a}_e$ are locally feasible given \algname\ assignments of $\hat{\mathbf{c}}_\ep \geq 0$ for all $\ep$.
\end{proposition}

\begin{proof}
By construction.
Consider any decision epoch $\ep$.
As given, consider any \algname\ option $\hat{a}_\ep$.
By definition of \algname, it assigns $\hat{\mathbf{c}}_\ep = \left[\frac{\hat{\mathbf{c}} - \mathbf{C}_{\ep-1}}{\gamma^{t_{\ep}}}\right]^+ \geq 0$.
By~\cref{prop:local-feasible-option-with-hatct-one-step-globally-feasible}, it is one-step globally feasible.
Since this is true for all $\ep$, \algname\ is globally feasible.
\end{proof}

%% file: algs/cobt.tex
\begin{algorithm}[H]
    \caption{Constrained Options Belief Tree Search} \label{alg:cobets}
    \begin{algorithmic}[1]
                \Procedure{SelectOption}{$b,\hat{\vect{c}}$}
            \State $\vect{\lambda} \gets \vect{\lambda}_0$
            \For{$i \in 1:n$}
                \State $Q_\lambda(b\blue{\hat{a}}) := Q(b\blue{\hat{a}})-\vect{\lambda}^\top \vect{Q}_C(b\blue{\hat{a}})$
                \State $\Call{Simulate}{b, \blue{\hat{\vect{c}}}, d_\text{max}}$
                \State $\blue{\hat{a}} \gets \argmax_{\blue{\hat{a}}}  Q_\lambda(b\blue{\hat{a}})$
                \State $\vect{\lambda} \gets [\vect{\lambda} + \alpha_i (\vect{Q}_C(b\blue{\hat{a}})-\hat{\vect{c}})]^+$
            \EndFor
            \State $\textbf{return } \argmax_{\blue{\hat{a}}}  Q(b\blue{\hat{a}}) \text{ s.t. } \vect{Q}_C(b\blue{\hat{a}}) \leq \hat{\vect{c}}$
        \EndProcedure
        \Procedure {OptionProgWiden}{$b, \blue{\hat{\vect{c}}}$}
            \If{$|C(b)| \leq k_a N(b)^{\alpha_a}$}
                \State $\blue{\hat{a} \gets \Call{SampleNextOption}{b, [\hat{\vect{c}}]^+}}$
                \State $C(b) \gets C(b) \cup \{\blue{\hat{a}}\}$
            \EndIf
            \State $Q_{\lambda U C B}(b\blue{\hat{a}}) := Q_{\lambda}(b\blue{\hat{a}}) + \kappa \sqrt{\frac{\log N(b)}{N(b\blue{\hat{a}})}}$
            \State $\textbf{return } \argmax_{\blue{\hat{a}}}  Q_{\lambda U C B}(b\blue{\hat{a}})$
        \EndProcedure
        \Procedure {Simulate}{$b$, $\blue{\hat{\vect{c}}}$, $d$}        
            \If{$d \leq 0$}
                \State \textbf{return} $0, \vect{0}$
            \EndIf
            \State $\blue{\hat{a} \gets \Call{OptionProgWiden}{b, \hat{\vect{c}}}}$
            \If{$|C(b\blue{\hat{a}})| \leq k_o N(b\blue{\hat{a}})^{\alpha_o}$}
                \State $\blue{b',\tilde{r},\tilde{\vect{c}} \gets G_\text{PF($m$)}(b,\Call{Action}{\hat{a},b})}$
                \State $\blue{\tau \gets 1}$
                \While $\blue{\neg \Call{Terminate}{\hat{a},b'} \land (d-\tau > 0)}$
                    \State $\blue{b',r,\vect{c} \gets G_\text{PF($m$)}(b',\Call{Action}{\hat{a},b'})}$
                    \State $\blue{\tilde{r} \gets \tilde{r} + \gamma^\tau r}$
                    \State $\blue{\tilde{\vect{c}} \gets \tilde{\vect{c}} + \gamma^\tau \vect{c}}$
                    \State $\blue{\tau \gets \tau + 1}$
                \EndWhile
                \State $C(b\blue{\hat{a}}) \gets C(b\blue{\hat{a}}) \cup \{(\blue{b',\tilde{r},\tilde{\vect{c}}, \tau})\}$
                \State $V',\vect{C}' \gets \Call{EstimateValue}{b', \blue{\frac{\hat{\vect{c}}-\tilde{\vect{c}}}{\gamma^\tau}},d-\blue{\tau}}$
            \Else
                \State $\blue{b',\tilde{r},\tilde{\vect{c}}, \tau} \gets \text{sample uniformly from } C(b\blue{\hat{a}})$
                \State $V',\vect{C}' \gets \Call{Simulate}{b', \blue{\frac{\hat{\vect{c}}-\tilde{\vect{c}}}{\gamma^\tau}}, d-\blue{\tau}}$
            \EndIf
            \State $V \gets \blue{\tilde{r} + \gamma^\tau V'}$
            \State $\vect{C} \gets \blue{\tilde{\vect{c}} + \gamma^\tau \vect{C}'}$
            \State $N(b) \gets N(b)+1$
            \State $N(b\blue{\hat{a}}) \gets N(b\blue{\hat{a}})+1$
            \State $Q(b\blue{\hat{a}}) \gets Q(b\blue{\hat{a}}) + \frac{V - Q(b\blue{\hat{a}})}{N(b\blue{\hat{a}})}$
            \State $\vect{Q}_C(b\blue{\hat{a}}) \gets \vect{Q}_C(b\blue{\hat{a}}) + \frac{\vect{C} - \vect{Q}_C(b\blue{\hat{a}})}{N(b\blue{\hat{a}})}$
            \State \textbf{return} $V, \vect{C}$
        \EndProcedure
    \end{algorithmic}
\end{algorithm}

%% file: sections/04-experiments.tex
\section{Experiments}

Our experiments consider online planning in four large safety-critical, partially observable planning problems in order to empirically demonstrate the efficacy of \algname. We compare \algname\ against different non-hierarchical solvers on our target domains, investigate its anytime properties, and show that it can yield better plans even when the number of options far exceeds the number of underlying actions. 

We use Julia 1.9 and the POMDPs.jl framework for our experiments~\cite{egorov2017pomdps}. In the following sections, we outline the CPOMDP target problems, briefly describe their hierarchical decompositions, and discuss the main results from our experiments. For full experimentation details, including CPOMDP modeling details, the precise options crafted, and choices of hyperparameters, please refer to our code, which we have open-sourced at \texttt{github.com/sisl/COBTSExperiments}.

\subsection{CPOMDP Problems and Option Policies}
 
We highlight the CPOMDP problem domains used in our experiments below, whether their state, action, and observation spaces are (D)iscrete or (C)ontinuous, and provide an overview of the types of options crafted for execution. 

\subsubsection{Constrained LightDark~\cite{jamgochian2023online} (C, D, C)} 

In this one-dimensional robot localization problem adapted from LightDark \cite{platt2010belief,sunberg2018online}, the robot must first safely localize itself before navigating to the goal. The robot can move in discrete steps of $\mathcal{A} = \{0, \pm 1, \pm 5, \pm 10\}$ in order to navigate to $s\in[-1,1]$, take action $0$, and receive $+100$ reward, but taking action $0$ elsewhere accrues a $-100$ reward. The robot accrues a per-step reward of $-1$. The agent starts in the dark region, $b_0 = \mathcal{N}(2,2^2)$, and can navigate towards the light region at $s=10$ to help localize itself with less noisy observations. However, the robot must avoid entering a constraint region above $s=12$ where it will receive a per-step cost of $1$ and violate a budget of $\hat{c}=0.1$. As such, taking the $+10$ action immediately would violate the constraint in expectation. 

We template four types of options for this problem. \texttt{GoToGoal} greedily navigates the robot's mean position to the goal and terminates. \texttt{LocalizeFast} greedily navigates the robot's mean position to the light region until the belief uncertainty is sufficiently small. \texttt{LocalizeFromBelow} adjusts the navigation technique for localization so that the robot's mean position does not overshoot the light region. \texttt{LocalizeSafe} uses the robot's position uncertainty while localizing to minimize the risk that the robot violates the constraint. 

\subsubsection{Constrained Spillpoint~\cite{jamgochian2023online} (C, C, C)} This problem models safe geological carbon capture and sequestration around uncertain subsurface geometries and properties. In the original POMDP~\cite{spillpoint}, instances of CO$_2$ leaking through faults in the geometry are heavily penalized, both for the presence of a leak and for the total amount leaked. The constrained adaptation imposes a constraint of $\hat{c}=\num{1e-6}$ to ensure minimal CO$_2$ leakage.

The options for the spillpoint problem include \texttt{InferGeology} and \texttt{SafeFill}. The \texttt{InferGeology} begins by injecting 90\% of the CO$_2$ volume of the lowest-volume instance of the geology according to the current belief. Then, a sequence of observations of conducted which provides information on the shape of the geology (these observations indicate the presence of CO$_2$ at various spatial locations). The \texttt{InferGeology} is templated so a variety of observation sequences can be selected. The \texttt{SafeFill} option involves injecting 90\% of the CO$_2$ volume of the lowest-volume instance of the geology according to the current belief and then terminating the episode. The options (with five versions of \texttt{InferGeology}) were combined with the standard set of five individual actions for a total of \num{11} options.

\subsubsection{Constrained Bumper Roomba (C, D, D)} 

Roomba models a robot with an uncertain initial pose in a fixed environment as it uses its sensors to navigate to a goal region while avoiding a penalty region~\cite{wu2021adaptive}. In this work, we augment the problem to include a constraint region that the robot must avoid traveling through as it navigates to the goal. States are defined by the continuous pose of the robot on the map, actions allow the robot to turn or move by discrete amounts, and while the robot does not have access to its true pose, in Bumper Roomba, it receives a binary observation when it collides with a wall. 

Bumper Roomba crafts three types of options: \texttt{TurnAndGo} options turn the robot a fixed amount then navigate until the robot collides with a wall, \texttt{GreedyGoToGoal} greedily navigates to the goal using the robot's mean pose, and \texttt{SafeGoToGoal} navigates to the goal while imposing a barrier function around the constraint region.

\subsubsection{Constrained Lidar Roomba (C, D, C)} 
This CPOMDP augments Constrained Bumper Roomba with a Lidar sensor that noisily observes the distance to the nearest wall along the robot's heading, with noise proportional to the distance. Rather than using \texttt{TurnAndGo} options for localization, Constrained Lidar Roomba implements \texttt{Spin} options that turn the robot for different periods of time and with different turn radii to localize.

\subsection{Experiments and Discussion}

\begin{table*}[htpb]
\ra{1.2}
\centering
\scalebox{0.95}{\input{tables/experiments.tex}}
\caption{Online CPOMDP algorithm demonstrations comparing mean discounted cumulative rewards ($\hat{V}_R$) and costs ($\hat{V}_C$) across 100 LightDark simulations, 10 Spillpoint simulations, and 50 Bumper and Lidar Roomba simulations. \algname\ consistently satisfies constraints while outperforming both the CPFT-DPW baseline and the CPOMCPOW baseline.}
\label{table:results}
\vspace{-0.6cm}
\end{table*}

\subsubsection{CPOMDP algorithm comparison}
To evaluate \algname, we measure the mean discounted cumulative reward and cost for different planning episodes on the aforementioned CPOMDP domains. We benchmark \algname\ against the closest non-hierarchical online CPOMDP planning algorithms, CPOMCPOW and CPFT-DPW~\cite{jamgochian2023online}, which perform Lagrangian MCTS on the state spaces and belief-state spaces respectively. \Cref{table:results} summarizes the performance of the algorithms on the different CPOMDP domains, averaged across 100 LightDark, 10 Spillpoint, and 50 Roomba simulations.

In summary, we see that \algname\ is able to significantly outperform against baselines on all domains while satisfying cost constraints. In both Roomba problems, baselines are unable to search deep enough to localize and get to the goal, and instead meander while accruing step penalties to avoid the risk of violating the constraint.

\subsubsection{Anytime constraint satisfaction}
\begin{figure}[t]
\centerline{\input{figs/sweep_queries_50}}
\vspace{-0.3cm}
\caption{Mean cumulative discounted rewards (above) and costs (below) vs. number of tree queries across 50 Constrained LightDark simulations when using \algname\ with feasible options. \algname\ stays safe anytime while CPFT-DPW only satisfies constraints in the limit.}
\label{fig:anytimesafety}
\vspace{-0.3cm}
\end{figure}
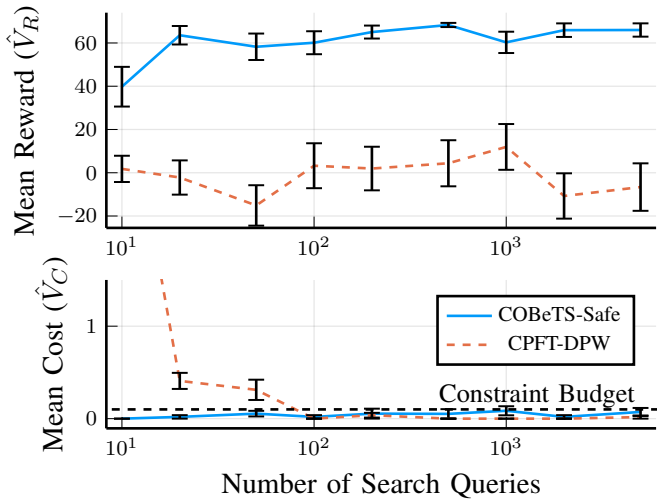

To highlight the anytime constraint satisfaction with \algname\, we vary the number of search queries in the Constrained LightDark CPOMDP with the robot restricted to \texttt{GoToGoal} and two different \texttt{LocalizeSafe} options, all of which are feasible from the initial belief. The results averaged across 50 simulations are depicted in~\Cref{fig:anytimesafety}. We see that even with low numbers of search queries, \algname\ satisfies the constraints while achieving high reward, while CPFT-DPW only satisfies the constraints as the number of queries is increased.

\subsubsection{Searching over many options}
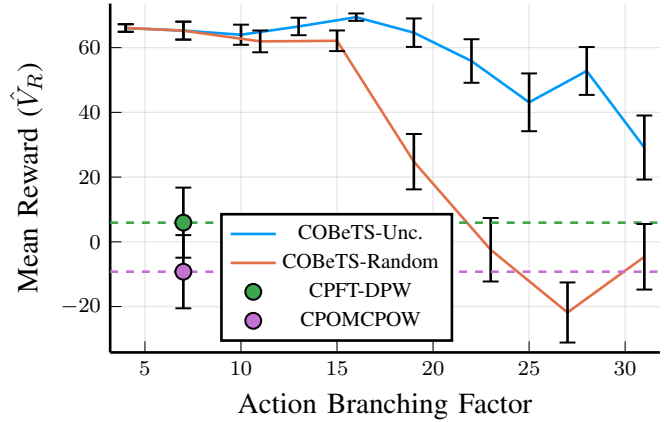
\begin{figure}[htpb]
\centerline{\input{figs/sweep_options_50.tex}}
\vspace{-0.3cm}
\caption{Mean cumulative discounted rewards for different numbers of options averaged over 50 Constrained LightDark simulations. All costs are feasible (not shown). \algname\ can retain high reward at larger action branching factors because hierarchy induces a smaller overall tree size.}
\label{fig:branching}
\vspace{-0.3cm}
\end{figure}

Finally, we investigate the impact that the action branching factor has on search quality in LightDark. We vary the branching factor in the \algname\ search by adding options using different strategies. \algname-Unc. adds options that localize to different sampled uncertainties, while \algname-Random adds options that execute three randomly selected non-terminal actions. 
The results, depicted in \Cref{fig:branching}, show that \algname\ still finds quality plans with a significant number of options to search over, much greater than the number of actions in the underlying problem (seven). These results support the analysis presented in~\Cref{eq:treesize}, that search complexity reduction from hierarchical decomposition can compensate for increased action branching.

%% file: tables/experiments.tex
\begin{tabular}{@{}r RRRRRRRR @{}}
\toprule
&\multicolumn{2}{c}{\underline{\quad\quad\textbf{LightDark}\quad\quad}}
&\multicolumn{2}{c}{\underline{\quad\quad\textbf{Spillpoint}\quad\quad}}
&\multicolumn{2}{c}{\underline{\quad\quad\quad\textbf{Bumper Roomba}\quad\quad\quad}}
&\multicolumn{2}{c}{\underline{\quad\quad\quad\textbf{Lidar Roomba}\quad\quad\quad}}\\
\textbf{Model} &  
{\hat{V}_R}  & {\hat{V}_C\ [\leq 0.1]} & 
{\hat{V}_R}  & {\hat{V}_C\ [\leq 10^{-6}]} & 
{\hat{V}_R}  & {\hat{V}_C\ [\leq 0.1]} &
{\hat{V}_R}  & {\hat{V}_C\ [\leq 0.1]} \\
\midrule
\texttt{\algname}           & 
\mathbf{68.6} \scriptstyle\pm 0.7 & 0.027 \scriptstyle\pm 0.015  & 
\mathbf{3.40} \scriptstyle\pm 0.66 & 0.000 \scriptstyle\pm 0.000  & 
\mathbf{5.73} \scriptstyle\pm 0.67 & 0.038 \scriptstyle\pm 0.038  & 
\mathbf{5.23} \scriptstyle\pm 0.67 & 0.020 \scriptstyle\pm 0.019  \\ 
\texttt{CPFT-DPW}            & 
5.9 \scriptstyle\pm 7.7 & 0.000 \scriptstyle\pm 0.000  &  
1.50 \scriptstyle\pm 0.39 & 0.000 \scriptstyle\pm 0.000  & 
-4.76 \scriptstyle\pm 0.00 & 0.036 \scriptstyle\pm 0.036  & 
-4.57 \scriptstyle\pm 0.19 & 0.000 \scriptstyle\pm 0.000  \\ 
\texttt{CPOMCPOW}            & 
-9.2 \scriptstyle\pm 8.0 & 0.032 \scriptstyle\pm 0.015  & 
1.51 \scriptstyle\pm 0.40 & 1.0 \cdot 10^{-5} \scriptstyle\pm 0.9 \cdot 10^{-5}  & 
-3.76 \scriptstyle\pm 0.42 & 0.000 \scriptstyle\pm 0.000  & 
-4.55 \scriptstyle\pm 0.20 & 0.000 \scriptstyle\pm 0.000  \\ 
\bottomrule
\end{tabular}


%% file: figs/sweep_queries_50.tex

\begin{tikzpicture}[/tikz/background rectangle/.style={fill={rgb,1:red,1.0;green,1.0;blue,1.0}, fill opacity={1.0}, draw opacity={1.0}}, show background rectangle]
\begin{axis}[point meta max={nan}, point meta min={nan}, legend cell align={left}, legend columns={1}, title={}, title style={at={{(0.5,1)}}, anchor={south}, font={{\fontsize{14 pt}{18.2 pt}\selectfont}}, color={rgb,1:red,0.0;green,0.0;blue,0.0}, draw opacity={1.0}, rotate={0.0}, align={center}}, 
legend style={color={rgb,1:red,0.0;green,0.0;blue,0.0}, draw opacity={1.0}, line width={1}, solid, fill={rgb,1:red,1.0;green,1.0;blue,1.0}, fill opacity={1.0}, text opacity={1.0}, font={{\fontsize{8 pt}{10.4 pt}\selectfont}}, text={rgb,1:red,0.0;green,0.0;blue,0.0}, cells={anchor={center}}, at={(0.4, 0.5)}, anchor={north west}}, 
axis background/.style={fill={rgb,1:red,1.0;green,1.0;blue,1.0}, opacity={1.0}}, anchor={north west}, 
xshift=0, yshift=0, 
width=0.5\textwidth, height=0.25\textwidth, 
scaled x ticks={false}, xlabel={}, x tick style={color={rgb,1:red,0.0;green,0.0;blue,0.0}, opacity={1.0}}, x tick label style={color={rgb,1:red,0.0;green,0.0;blue,0.0}, opacity={1.0}, rotate={0}}, xlabel style={at={(ticklabel cs:0.5)}, anchor=near ticklabel, at={{(ticklabel cs:0.5)}}, anchor={near ticklabel}, font={{\fontsize{11 pt}{14.3 pt}\selectfont}}, color={rgb,1:red,0.0;green,0.0;blue,0.0}, draw opacity={1.0}, rotate={0.0}}, xmode={log}, log basis x={10}, xmajorgrids={true}, xmin={8.299098131483817}, xmax={6024.751028104829}, xticklabels={{$10^{1}$,$10^{2}$,$10^{3}$}}, xtick={{10.0,100.0,1000.0}}, xtick align={inside}, xticklabel style={font={{\fontsize{8 pt}{10.4 pt}\selectfont}}, color={rgb,1:red,0.0;green,0.0;blue,0.0}, draw opacity={1.0}, rotate={0.0}}, x grid style={color={rgb,1:red,0.0;green,0.0;blue,0.0}, draw opacity={0.1}, line width={0.5}, solid}, axis x line*={left}, x axis line style={color={rgb,1:red,0.0;green,0.0;blue,0.0}, draw opacity={1.0}, line width={1}, solid}, scaled y ticks={false}, ylabel={Mean Reward ($\hat{V}_R$)}, y tick style={color={rgb,1:red,0.0;green,0.0;blue,0.0}, opacity={1.0}}, y tick label style={color={rgb,1:red,0.0;green,0.0;blue,0.0}, opacity={1.0}, rotate={0}}, ylabel style={at={(ticklabel cs:0.5)}, anchor=near ticklabel, at={{(ticklabel cs:0.5)}}, anchor={near ticklabel}, font={{\fontsize{11 pt}{14.3 pt}\selectfont}}, color={rgb,1:red,0.0;green,0.0;blue,0.0}, draw opacity={1.0}, rotate={0.0}}, ymajorgrids={true}, ymin={-25.5}, ymax={74.0047896312046}, yticklabels={{$-20$,$0$,$20$,$40$,$60$}}, ytick={{-20.0,0.0,20.0,40.0,60.0}}, ytick align={inside}, yticklabel style={font={{\fontsize{8 pt}{10.4 pt}\selectfont}}, color={rgb,1:red,0.0;green,0.0;blue,0.0}, draw opacity={1.0}, rotate={0.0}}, y grid style={color={rgb,1:red,0.0;green,0.0;blue,0.0}, draw opacity={0.1}, line width={0.5}, solid}, axis y line*={left}, y axis line style={color={rgb,1:red,0.0;green,0.0;blue,0.0}, draw opacity={1.0}, line width={1}, solid}, colorbar={false}]
    \addplot[color={rgb,1:red,0.0;green,0.6056;blue,0.9787}, name path={1b949742-b9b1-444b-b3b5-1999e585da09}, draw opacity={1.0}, line width={1}, solid]
        table[row sep={\\}]
        {
            \\
            10.0  39.77870737477656  \\
            20.0  63.57204728558905  \\
            50.0  58.243725540812484  \\
            100.0  60.11069227313138  \\
            200.0  65.04596415100178  \\
            500.0  68.31774518076766  \\
            1000.0  60.280308007110925  \\
            2000.0  65.92364021718281  \\
            5000.0  65.99273478237029  \\
        }
        ;
    \addplot[color={rgb,1:red,0.0;green,0.0;blue,0.0}, name path={6bf628a9-f189-4c39-b4d2-99c157991f1d}, draw opacity={1.0}, line width={1}, solid, mark={-}, mark size={3.0 pt}, mark repeat={1}, mark options={color={rgb,1:red,0.0;green,0.0;blue,0.0}, draw opacity={1.0}, fill={rgb,1:red,0.0;green,0.0;blue,0.0}, fill opacity={1.0}, line width={0.75}, rotate={0}, solid}]
        table[row sep={\\}]
        {
            \\
            10.0  30.599387808726114  \\
            10.0  48.958026940827004  \\
        }
        ;
    \addplot[color={rgb,1:red,0.0;green,0.0;blue,0.0}, name path={6bf628a9-f189-4c39-b4d2-99c157991f1d}, draw opacity={1.0}, line width={1}, solid, mark={-}, mark size={3.0 pt}, mark repeat={1}, mark options={color={rgb,1:red,0.0;green,0.0;blue,0.0}, draw opacity={1.0}, fill={rgb,1:red,0.0;green,0.0;blue,0.0}, fill opacity={1.0}, line width={0.75}, rotate={0}, solid}]
        table[row sep={\\}]
        {
            \\
            20.0  59.29711821757667  \\
            20.0  67.84697635360143  \\
        }
        ;
    \addplot[color={rgb,1:red,0.0;green,0.0;blue,0.0}, name path={6bf628a9-f189-4c39-b4d2-99c157991f1d}, draw opacity={1.0}, line width={1}, solid, mark={-}, mark size={3.0 pt}, mark repeat={1}, mark options={color={rgb,1:red,0.0;green,0.0;blue,0.0}, draw opacity={1.0}, fill={rgb,1:red,0.0;green,0.0;blue,0.0}, fill opacity={1.0}, line width={0.75}, rotate={0}, solid}]
        table[row sep={\\}]
        {
            \\
            50.0  52.13340985607103  \\
            50.0  64.35404122555394  \\
        }
        ;
    \addplot[color={rgb,1:red,0.0;green,0.0;blue,0.0}, name path={6bf628a9-f189-4c39-b4d2-99c157991f1d}, draw opacity={1.0}, line width={1}, solid, mark={-}, mark size={3.0 pt}, mark repeat={1}, mark options={color={rgb,1:red,0.0;green,0.0;blue,0.0}, draw opacity={1.0}, fill={rgb,1:red,0.0;green,0.0;blue,0.0}, fill opacity={1.0}, line width={0.75}, rotate={0}, solid}]
        table[row sep={\\}]
        {
            \\
            100.0  54.803081130300875  \\
            100.0  65.41830341596189  \\
        }
        ;
    \addplot[color={rgb,1:red,0.0;green,0.0;blue,0.0}, name path={6bf628a9-f189-4c39-b4d2-99c157991f1d}, draw opacity={1.0}, line width={1}, solid, mark={-}, mark size={3.0 pt}, mark repeat={1}, mark options={color={rgb,1:red,0.0;green,0.0;blue,0.0}, draw opacity={1.0}, fill={rgb,1:red,0.0;green,0.0;blue,0.0}, fill opacity={1.0}, line width={0.75}, rotate={0}, solid}]
        table[row sep={\\}]
        {
            \\
            200.0  62.04810936361468  \\
            200.0  68.04381893838888  \\
        }
        ;
    \addplot[color={rgb,1:red,0.0;green,0.0;blue,0.0}, name path={6bf628a9-f189-4c39-b4d2-99c157991f1d}, draw opacity={1.0}, line width={1}, solid, mark={-}, mark size={3.0 pt}, mark repeat={1}, mark options={color={rgb,1:red,0.0;green,0.0;blue,0.0}, draw opacity={1.0}, fill={rgb,1:red,0.0;green,0.0;blue,0.0}, fill opacity={1.0}, line width={0.75}, rotate={0}, solid}]
        table[row sep={\\}]
        {
            \\
            500.0  67.35210304605793  \\
            500.0  69.28338731547738  \\
        }
        ;
    \addplot[color={rgb,1:red,0.0;green,0.0;blue,0.0}, name path={6bf628a9-f189-4c39-b4d2-99c157991f1d}, draw opacity={1.0}, line width={1}, solid, mark={-}, mark size={3.0 pt}, mark repeat={1}, mark options={color={rgb,1:red,0.0;green,0.0;blue,0.0}, draw opacity={1.0}, fill={rgb,1:red,0.0;green,0.0;blue,0.0}, fill opacity={1.0}, line width={0.75}, rotate={0}, solid}]
        table[row sep={\\}]
        {
            \\
            1000.0  55.35336245826721  \\
            1000.0  65.20725355595464  \\
        }
        ;
    \addplot[color={rgb,1:red,0.0;green,0.0;blue,0.0}, name path={6bf628a9-f189-4c39-b4d2-99c157991f1d}, draw opacity={1.0}, line width={1}, solid, mark={-}, mark size={3.0 pt}, mark repeat={1}, mark options={color={rgb,1:red,0.0;green,0.0;blue,0.0}, draw opacity={1.0}, fill={rgb,1:red,0.0;green,0.0;blue,0.0}, fill opacity={1.0}, line width={0.75}, rotate={0}, solid}]
        table[row sep={\\}]
        {
            \\
            2000.0  62.78901013647667  \\
            2000.0  69.05827029788895  \\
        }
        ;
    \addplot[color={rgb,1:red,0.0;green,0.0;blue,0.0}, name path={6bf628a9-f189-4c39-b4d2-99c157991f1d}, draw opacity={1.0}, line width={1}, solid, mark={-}, mark size={3.0 pt}, mark repeat={1}, mark options={color={rgb,1:red,0.0;green,0.0;blue,0.0}, draw opacity={1.0}, fill={rgb,1:red,0.0;green,0.0;blue,0.0}, fill opacity={1.0}, line width={0.75}, rotate={0}, solid}]
        table[row sep={\\}]
        {
            \\
            5000.0  62.9193659590075  \\
            5000.0  69.06610360573308  \\
        }
        ;
    \addplot[color={rgb,1:red,0.8889;green,0.4356;blue,0.2781}, name path={07e29189-4675-491c-bf39-01b7e34e0313}, draw opacity={1.0}, line width={1}, dashed]
        table[row sep={\\}]
        {
            \\
            10.0  1.7878042862233963  \\
            20.0  -2.2035832372205757  \\
            50.0  -15.115385247558793  \\
            100.0  3.2435340843474187  \\
            200.0  1.9467625510198596  \\
            500.0  4.381294437230879  \\
            1000.0  11.940531108442837  \\
            2000.0  -10.73708411149656  \\
            5000.0  -6.626790359700762  \\
        }
        ;
    \addplot[color={rgb,1:red,0.0;green,0.0;blue,0.0}, name path={fd94b716-574c-402a-80cc-e383b3e892d3}, draw opacity={1.0}, line width={1}, solid, mark={-}, mark size={3.0 pt}, mark repeat={1}, mark options={color={rgb,1:red,0.0;green,0.0;blue,0.0}, draw opacity={1.0}, fill={rgb,1:red,0.0;green,0.0;blue,0.0}, fill opacity={1.0}, line width={0.75}, rotate={0}, solid}]
        table[row sep={\\}]
        {
            \\
            10.0  -4.2820845892340245  \\
            10.0  7.857693161680817  \\
        }
        ;
    \addplot[color={rgb,1:red,0.0;green,0.0;blue,0.0}, name path={fd94b716-574c-402a-80cc-e383b3e892d3}, draw opacity={1.0}, line width={1}, solid, mark={-}, mark size={3.0 pt}, mark repeat={1}, mark options={color={rgb,1:red,0.0;green,0.0;blue,0.0}, draw opacity={1.0}, fill={rgb,1:red,0.0;green,0.0;blue,0.0}, fill opacity={1.0}, line width={0.75}, rotate={0}, solid}]
        table[row sep={\\}]
        {
            \\
            20.0  -10.121540621646405  \\
            20.0  5.714374147205254  \\
        }
        ;
    \addplot[color={rgb,1:red,0.0;green,0.0;blue,0.0}, name path={fd94b716-574c-402a-80cc-e383b3e892d3}, draw opacity={1.0}, line width={1}, solid, mark={-}, mark size={3.0 pt}, mark repeat={1}, mark options={color={rgb,1:red,0.0;green,0.0;blue,0.0}, draw opacity={1.0}, fill={rgb,1:red,0.0;green,0.0;blue,0.0}, fill opacity={1.0}, line width={0.75}, rotate={0}, solid}]
        table[row sep={\\}]
        {
            \\
            50.0  -24.45926310905281  \\
            50.0  -5.771507386064776  \\
        }
        ;
    \addplot[color={rgb,1:red,0.0;green,0.0;blue,0.0}, name path={fd94b716-574c-402a-80cc-e383b3e892d3}, draw opacity={1.0}, line width={1}, solid, mark={-}, mark size={3.0 pt}, mark repeat={1}, mark options={color={rgb,1:red,0.0;green,0.0;blue,0.0}, draw opacity={1.0}, fill={rgb,1:red,0.0;green,0.0;blue,0.0}, fill opacity={1.0}, line width={0.75}, rotate={0}, solid}]
        table[row sep={\\}]
        {
            \\
            100.0  -7.142442818540676  \\
            100.0  13.629510987235513  \\
        }
        ;
    \addplot[color={rgb,1:red,0.0;green,0.0;blue,0.0}, name path={fd94b716-574c-402a-80cc-e383b3e892d3}, draw opacity={1.0}, line width={1}, solid, mark={-}, mark size={3.0 pt}, mark repeat={1}, mark options={color={rgb,1:red,0.0;green,0.0;blue,0.0}, draw opacity={1.0}, fill={rgb,1:red,0.0;green,0.0;blue,0.0}, fill opacity={1.0}, line width={0.75}, rotate={0}, solid}]
        table[row sep={\\}]
        {
            \\
            200.0  -8.13358981266386  \\
            200.0  12.027114914703578  \\
        }
        ;
    \addplot[color={rgb,1:red,0.0;green,0.0;blue,0.0}, name path={fd94b716-574c-402a-80cc-e383b3e892d3}, draw opacity={1.0}, line width={1}, solid, mark={-}, mark size={3.0 pt}, mark repeat={1}, mark options={color={rgb,1:red,0.0;green,0.0;blue,0.0}, draw opacity={1.0}, fill={rgb,1:red,0.0;green,0.0;blue,0.0}, fill opacity={1.0}, line width={0.75}, rotate={0}, solid}]
        table[row sep={\\}]
        {
            \\
            500.0  -6.271232235570458  \\
            500.0  15.033821110032214  \\
        }
        ;
    \addplot[color={rgb,1:red,0.0;green,0.0;blue,0.0}, name path={fd94b716-574c-402a-80cc-e383b3e892d3}, draw opacity={1.0}, line width={1}, solid, mark={-}, mark size={3.0 pt}, mark repeat={1}, mark options={color={rgb,1:red,0.0;green,0.0;blue,0.0}, draw opacity={1.0}, fill={rgb,1:red,0.0;green,0.0;blue,0.0}, fill opacity={1.0}, line width={0.75}, rotate={0}, solid}]
        table[row sep={\\}]
        {
            \\
            1000.0  1.3562917415389055  \\
            1000.0  22.52477047534677  \\
        }
        ;
    \addplot[color={rgb,1:red,0.0;green,0.0;blue,0.0}, name path={fd94b716-574c-402a-80cc-e383b3e892d3}, draw opacity={1.0}, line width={1}, solid, mark={-}, mark size={3.0 pt}, mark repeat={1}, mark options={color={rgb,1:red,0.0;green,0.0;blue,0.0}, draw opacity={1.0}, fill={rgb,1:red,0.0;green,0.0;blue,0.0}, fill opacity={1.0}, line width={0.75}, rotate={0}, solid}]
        table[row sep={\\}]
        {
            \\
            2000.0  -21.214106048415204  \\
            2000.0  -0.2600621745779179  \\
        }
        ;
    \addplot[color={rgb,1:red,0.0;green,0.0;blue,0.0}, name path={fd94b716-574c-402a-80cc-e383b3e892d3}, draw opacity={1.0}, line width={1}, solid, mark={-}, mark size={3.0 pt}, mark repeat={1}, mark options={color={rgb,1:red,0.0;green,0.0;blue,0.0}, draw opacity={1.0}, fill={rgb,1:red,0.0;green,0.0;blue,0.0}, fill opacity={1.0}, line width={0.75}, rotate={0}, solid}]
        table[row sep={\\}]
        {
            \\
            5000.0  -17.603457546943094  \\
            5000.0  4.3498768275415705  \\
        }
        ;
\end{axis}
\begin{axis}[point meta max={nan}, point meta min={nan}, legend cell align={left}, legend columns={1}, title={}, title style={at={{(0.5,1)}}, anchor={south}, font={{\fontsize{14 pt}{18.2 pt}\selectfont}}, color={rgb,1:red,0.0;green,0.0;blue,0.0}, draw opacity={1.0}, rotate={0.0}, align={center}}, 
legend style={color={rgb,1:red,0.0;green,0.0;blue,0.0}, draw opacity={1.0}, line width={1}, solid, fill={rgb,1:red,1.0;green,1.0;blue,1.0}, fill opacity={1.0}, text opacity={1.0}, font={{\fontsize{8 pt}{10.4 pt}\selectfont}}, text={rgb,1:red,0.0;green,0.0;blue,0.0}, cells={anchor={center}}, at={(0.6, 0.9)}, anchor={north west}}, axis background/.style={fill={rgb,1:red,1.0;green,1.0;blue,1.0}, opacity={1.0}}, anchor={north west}, 
xshift=0, yshift=-0.2\textwidth, 
width=0.5\textwidth, height=0.2\textwidth, 
scaled x ticks={false}, xlabel={Number of Search Queries}, x tick style={color={rgb,1:red,0.0;green,0.0;blue,0.0}, opacity={1.0}}, x tick label style={color={rgb,1:red,0.0;green,0.0;blue,0.0}, opacity={1.0}, rotate={0}}, xlabel style={at={(ticklabel cs:0.5)}, anchor=near ticklabel, at={{(ticklabel cs:0.5)}}, anchor={near ticklabel}, font={{\fontsize{11 pt}{14.3 pt}\selectfont}}, color={rgb,1:red,0.0;green,0.0;blue,0.0}, draw opacity={1.0}, rotate={0.0}}, xmode={log}, log basis x={10}, xmajorgrids={true}, xmin={8.299098131483817}, xmax={6024.751028104829}, xticklabels={{$10^{1}$,$10^{2}$,$10^{3}$}}, xtick={{10.0,100.0,1000.0}}, xtick align={inside}, xticklabel style={font={{\fontsize{8 pt}{10.4 pt}\selectfont}}, color={rgb,1:red,0.0;green,0.0;blue,0.0}, draw opacity={1.0}, rotate={0.0}}, x grid style={color={rgb,1:red,0.0;green,0.0;blue,0.0}, draw opacity={0.1}, line width={0.5}, solid}, axis x line*={left}, x axis line style={color={rgb,1:red,0.0;green,0.0;blue,0.0}, draw opacity={1.0}, line width={1}, solid}, scaled y ticks={false}, ylabel={Mean Cost ($\hat{V}_C$)}, y tick style={color={rgb,1:red,0.0;green,0.0;blue,0.0}, opacity={1.0}}, y tick label style={color={rgb,1:red,0.0;green,0.0;blue,0.0}, opacity={1.0}, rotate={0}}, ylabel style={at={(ticklabel cs:0.5)}, anchor=near ticklabel, at={{(ticklabel cs:0.5)}}, anchor={near ticklabel}, font={{\fontsize{11 pt}{14.3 pt}\selectfont}}, color={rgb,1:red,0.0;green,0.0;blue,0.0}, draw opacity={1.0}, rotate={0.0}}, ymajorgrids={true}, 
ymin={-0.10490091602235974}, ymax={1.5}, 
yticklabels={{$0$,$1$,$2$,$3$}}, ytick={{0.0,1.0,2.0,3.0}}, ytick align={inside}, yticklabel style={font={{\fontsize{8 pt}{10.4 pt}\selectfont}}, color={rgb,1:red,0.0;green,0.0;blue,0.0}, draw opacity={1.0}, rotate={0.0}}, y grid style={color={rgb,1:red,0.0;green,0.0;blue,0.0}, draw opacity={0.1}, line width={0.5}, solid}, axis y line*={left}, y axis line style={color={rgb,1:red,0.0;green,0.0;blue,0.0}, draw opacity={1.0}, line width={1}, solid}, colorbar={false}]
    \addplot[color={rgb,1:red,0.0;green,0.6056;blue,0.9787}, name path={a1546416-1f3a-420f-904b-1b8c71f93e79}, draw opacity={1.0}, line width={1}, solid]
        table[row sep={\\}]
        {
            \\
            10.0  0.0  \\
            20.0  0.019  \\
            50.0  0.054197499999999996  \\
            100.0  0.019  \\
            200.0  0.054197499999999996  \\
            500.0  0.051623118749999995  \\
            1000.0  0.08510586875  \\
            2000.0  0.019  \\
            5000.0  0.0722475  \\
        }
        ;
    \addlegendentry {\algname-Safe}
    \addplot[color={rgb,1:red,0.0;green,0.0;blue,0.0}, name path={4e54cfbf-8855-432c-87ef-3a05a3261a50}, draw opacity={1.0}, line width={1}, solid, mark={-}, mark size={3.0 pt}, mark repeat={1}, mark options={color={rgb,1:red,0.0;green,0.0;blue,0.0}, draw opacity={1.0}, fill={rgb,1:red,0.0;green,0.0;blue,0.0}, fill opacity={1.0}, line width={0.75}, rotate={0}, solid}, forget plot]
        table[row sep={\\}]
        {
            \\
            10.0  0.0  \\
            10.0  0.0  \\
        }
        ;
    \addplot[color={rgb,1:red,0.0;green,0.0;blue,0.0}, name path={4e54cfbf-8855-432c-87ef-3a05a3261a50}, draw opacity={1.0}, line width={1}, solid, mark={-}, mark size={3.0 pt}, mark repeat={1}, mark options={color={rgb,1:red,0.0;green,0.0;blue,0.0}, draw opacity={1.0}, fill={rgb,1:red,0.0;green,0.0;blue,0.0}, fill opacity={1.0}, line width={0.75}, rotate={0}, solid}, forget plot]
        table[row sep={\\}]
        {
            \\
            20.0  0.00019095962043785136  \\
            20.0  0.03780904037956215  \\
        }
        ;
    \addplot[color={rgb,1:red,0.0;green,0.0;blue,0.0}, name path={4e54cfbf-8855-432c-87ef-3a05a3261a50}, draw opacity={1.0}, line width={1}, solid, mark={-}, mark size={3.0 pt}, mark repeat={1}, mark options={color={rgb,1:red,0.0;green,0.0;blue,0.0}, draw opacity={1.0}, fill={rgb,1:red,0.0;green,0.0;blue,0.0}, fill opacity={1.0}, line width={0.75}, rotate={0}, solid}, forget plot]
        table[row sep={\\}]
        {
            \\
            50.0  0.023831534246791995  \\
            50.0  0.084563465753208  \\
        }
        ;
    \addplot[color={rgb,1:red,0.0;green,0.0;blue,0.0}, name path={4e54cfbf-8855-432c-87ef-3a05a3261a50}, draw opacity={1.0}, line width={1}, solid, mark={-}, mark size={3.0 pt}, mark repeat={1}, mark options={color={rgb,1:red,0.0;green,0.0;blue,0.0}, draw opacity={1.0}, fill={rgb,1:red,0.0;green,0.0;blue,0.0}, fill opacity={1.0}, line width={0.75}, rotate={0}, solid}, forget plot]
        table[row sep={\\}]
        {
            \\
            100.0  0.00019095962043785136  \\
            100.0  0.03780904037956215  \\
        }
        ;
    \addplot[color={rgb,1:red,0.0;green,0.0;blue,0.0}, name path={4e54cfbf-8855-432c-87ef-3a05a3261a50}, draw opacity={1.0}, line width={1}, solid, mark={-}, mark size={3.0 pt}, mark repeat={1}, mark options={color={rgb,1:red,0.0;green,0.0;blue,0.0}, draw opacity={1.0}, fill={rgb,1:red,0.0;green,0.0;blue,0.0}, fill opacity={1.0}, line width={0.75}, rotate={0}, solid}, forget plot]
        table[row sep={\\}]
        {
            \\
            200.0  0.0005447123172989468  \\
            200.0  0.10785028768270105  \\
        }
        ;
    \addplot[color={rgb,1:red,0.0;green,0.0;blue,0.0}, name path={4e54cfbf-8855-432c-87ef-3a05a3261a50}, draw opacity={1.0}, line width={1}, solid, mark={-}, mark size={3.0 pt}, mark repeat={1}, mark options={color={rgb,1:red,0.0;green,0.0;blue,0.0}, draw opacity={1.0}, fill={rgb,1:red,0.0;green,0.0;blue,0.0}, fill opacity={1.0}, line width={0.75}, rotate={0}, solid}, forget plot]
        table[row sep={\\}]
        {
            \\
            500.0  0.0005188384822272449  \\
            500.0  0.10272739901777275  \\
        }
        ;
    \addplot[color={rgb,1:red,0.0;green,0.0;blue,0.0}, name path={4e54cfbf-8855-432c-87ef-3a05a3261a50}, draw opacity={1.0}, line width={1}, solid, mark={-}, mark size={3.0 pt}, mark repeat={1}, mark options={color={rgb,1:red,0.0;green,0.0;blue,0.0}, draw opacity={1.0}, fill={rgb,1:red,0.0;green,0.0;blue,0.0}, fill opacity={1.0}, line width={0.75}, rotate={0}, solid}, forget plot]
        table[row sep={\\}]
        {
            \\
            1000.0  0.036070713166557175  \\
            1000.0  0.13414102433344283  \\
        }
        ;
    \addplot[color={rgb,1:red,0.0;green,0.0;blue,0.0}, name path={4e54cfbf-8855-432c-87ef-3a05a3261a50}, draw opacity={1.0}, line width={1}, solid, mark={-}, mark size={3.0 pt}, mark repeat={1}, mark options={color={rgb,1:red,0.0;green,0.0;blue,0.0}, draw opacity={1.0}, fill={rgb,1:red,0.0;green,0.0;blue,0.0}, fill opacity={1.0}, line width={0.75}, rotate={0}, solid}, forget plot]
        table[row sep={\\}]
        {
            \\
            2000.0  0.00019095962043785136  \\
            2000.0  0.03780904037956215  \\
        }
        ;
    \addplot[color={rgb,1:red,0.0;green,0.0;blue,0.0}, name path={4e54cfbf-8855-432c-87ef-3a05a3261a50}, draw opacity={1.0}, line width={1}, solid, mark={-}, mark size={3.0 pt}, mark repeat={1}, mark options={color={rgb,1:red,0.0;green,0.0;blue,0.0}, draw opacity={1.0}, fill={rgb,1:red,0.0;green,0.0;blue,0.0}, fill opacity={1.0}, line width={0.75}, rotate={0}, solid}, forget plot]
        table[row sep={\\}]
        {
            \\
            5000.0  0.028794678099403045  \\
            5000.0  0.11570032190059697  \\
        }
        ;
    \addplot[color={rgb,1:red,0.8889;green,0.4356;blue,0.2781}, name path={9f05310c-dcf6-492a-8ef1-d3b3a28ed9eb}, draw opacity={1.0}, line width={1}, dashed]
        table[row sep={\\}]
        {
            \\
            10.0  3.898630429189055  \\
            20.0  0.4082874976238083  \\
            50.0  0.3116880336291855  \\
            100.0  0.0  \\
            200.0  0.03705  \\
            500.0  0.0  \\
            1000.0  0.0  \\
            2000.0  0.0  \\
            5000.0  0.017147499999999996  \\
        }
        ;
    \addlegendentry {CPFT-DPW}
    \addplot[color={rgb,1:red,0.0;green,0.0;blue,0.0}, name path={0d9b4324-8fd5-4acb-9689-6126f39cc0ef}, draw opacity={1.0}, line width={1}, solid, mark={-}, mark size={3.0 pt}, mark repeat={1}, mark options={color={rgb,1:red,0.0;green,0.0;blue,0.0}, draw opacity={1.0}, fill={rgb,1:red,0.0;green,0.0;blue,0.0}, fill opacity={1.0}, line width={0.75}, rotate={0}, solid}, forget plot]
        table[row sep={\\}]
        {
            \\
            10.0  3.1293640368663764  \\
            10.0  4.667896821511733  \\
        }
        ;
    \addplot[color={rgb,1:red,0.0;green,0.0;blue,0.0}, name path={0d9b4324-8fd5-4acb-9689-6126f39cc0ef}, draw opacity={1.0}, line width={1}, solid, mark={-}, mark size={3.0 pt}, mark repeat={1}, mark options={color={rgb,1:red,0.0;green,0.0;blue,0.0}, draw opacity={1.0}, fill={rgb,1:red,0.0;green,0.0;blue,0.0}, fill opacity={1.0}, line width={0.75}, rotate={0}, solid}, forget plot]
        table[row sep={\\}]
        {
            \\
            20.0  0.3214326131886772  \\
            20.0  0.49514238205893935  \\
        }
        ;
    \addplot[color={rgb,1:red,0.0;green,0.0;blue,0.0}, name path={0d9b4324-8fd5-4acb-9689-6126f39cc0ef}, draw opacity={1.0}, line width={1}, solid, mark={-}, mark size={3.0 pt}, mark repeat={1}, mark options={color={rgb,1:red,0.0;green,0.0;blue,0.0}, draw opacity={1.0}, fill={rgb,1:red,0.0;green,0.0;blue,0.0}, fill opacity={1.0}, line width={0.75}, rotate={0}, solid}, forget plot]
        table[row sep={\\}]
        {
            \\
            50.0  0.20202529140045333  \\
            50.0  0.4213507758579177  \\
        }
        ;
    \addplot[color={rgb,1:red,0.0;green,0.0;blue,0.0}, name path={0d9b4324-8fd5-4acb-9689-6126f39cc0ef}, draw opacity={1.0}, line width={1}, solid, mark={-}, mark size={3.0 pt}, mark repeat={1}, mark options={color={rgb,1:red,0.0;green,0.0;blue,0.0}, draw opacity={1.0}, fill={rgb,1:red,0.0;green,0.0;blue,0.0}, fill opacity={1.0}, line width={0.75}, rotate={0}, solid}, forget plot]
        table[row sep={\\}]
        {
            \\
            100.0  0.0  \\
            100.0  0.0  \\
        }
        ;
    \addplot[color={rgb,1:red,0.0;green,0.0;blue,0.0}, name path={0d9b4324-8fd5-4acb-9689-6126f39cc0ef}, draw opacity={1.0}, line width={1}, solid, mark={-}, mark size={3.0 pt}, mark repeat={1}, mark options={color={rgb,1:red,0.0;green,0.0;blue,0.0}, draw opacity={1.0}, fill={rgb,1:red,0.0;green,0.0;blue,0.0}, fill opacity={1.0}, line width={0.75}, rotate={0}, solid}, forget plot]
        table[row sep={\\}]
        {
            \\
            200.0  0.011372218748497736  \\
            200.0  0.06272778125150226  \\
        }
        ;
    \addplot[color={rgb,1:red,0.0;green,0.0;blue,0.0}, name path={0d9b4324-8fd5-4acb-9689-6126f39cc0ef}, draw opacity={1.0}, line width={1}, solid, mark={-}, mark size={3.0 pt}, mark repeat={1}, mark options={color={rgb,1:red,0.0;green,0.0;blue,0.0}, draw opacity={1.0}, fill={rgb,1:red,0.0;green,0.0;blue,0.0}, fill opacity={1.0}, line width={0.75}, rotate={0}, solid}, forget plot]
        table[row sep={\\}]
        {
            \\
            500.0  0.0  \\
            500.0  0.0  \\
        }
        ;
    \addplot[color={rgb,1:red,0.0;green,0.0;blue,0.0}, name path={0d9b4324-8fd5-4acb-9689-6126f39cc0ef}, draw opacity={1.0}, line width={1}, solid, mark={-}, mark size={3.0 pt}, mark repeat={1}, mark options={color={rgb,1:red,0.0;green,0.0;blue,0.0}, draw opacity={1.0}, fill={rgb,1:red,0.0;green,0.0;blue,0.0}, fill opacity={1.0}, line width={0.75}, rotate={0}, solid}, forget plot]
        table[row sep={\\}]
        {
            \\
            1000.0  0.0  \\
            1000.0  0.0  \\
        }
        ;
    \addplot[color={rgb,1:red,0.0;green,0.0;blue,0.0}, name path={0d9b4324-8fd5-4acb-9689-6126f39cc0ef}, draw opacity={1.0}, line width={1}, solid, mark={-}, mark size={3.0 pt}, mark repeat={1}, mark options={color={rgb,1:red,0.0;green,0.0;blue,0.0}, draw opacity={1.0}, fill={rgb,1:red,0.0;green,0.0;blue,0.0}, fill opacity={1.0}, line width={0.75}, rotate={0}, solid}, forget plot]
        table[row sep={\\}]
        {
            \\
            2000.0  0.0  \\
            2000.0  0.0  \\
        }
        ;
    \addplot[color={rgb,1:red,0.0;green,0.0;blue,0.0}, name path={0d9b4324-8fd5-4acb-9689-6126f39cc0ef}, draw opacity={1.0}, line width={1}, solid, mark={-}, mark size={3.0 pt}, mark repeat={1}, mark options={color={rgb,1:red,0.0;green,0.0;blue,0.0}, draw opacity={1.0}, fill={rgb,1:red,0.0;green,0.0;blue,0.0}, fill opacity={1.0}, line width={0.75}, rotate={0}, solid}, forget plot]
        table[row sep={\\}]
        {
            \\
            5000.0  0.0001723410574451488  \\
            5000.0  0.03412265894255484  \\
        }
        ;
    \addplot[color={rgb,1:red,0.0;green,0.0;blue,0.0}, name path={be65e542-b0d7-472e-b8be-af97c69b805f}, draw opacity={1.0}, line width={1}, dashed, forget plot]
        table[row sep={\\}]
        {
            \\
            0.011432012621717218  0.1  \\
            4.373683064783866e6  0.1  \\
        }
        ;
    \node[right, , color={rgb,1:red,0.0;green,0.0;blue,0.0}, draw opacity={1.0}, rotate={0.0}, font={{\fontsize{10 pt}{13.0 pt}\selectfont}}]  at (axis cs:400,0.25) {Constraint Budget};
\end{axis}
\end{tikzpicture}

%% file: figs/sweep_options_50.tex

\begin{tikzpicture}[/tikz/background rectangle/.style={fill={rgb,1:red,1.0;green,1.0;blue,1.0}, fill opacity={1.0}, draw opacity={1.0}}, show background rectangle]
\begin{axis}[point meta max={nan}, point meta min={nan}, legend cell align={left}, legend columns={1}, title={}, title style={at={{(0.5,1)}}, anchor={south}, font={{\fontsize{14 pt}{18.2 pt}\selectfont}}, color={rgb,1:red,0.0;green,0.0;blue,0.0}, draw opacity={1.0}, rotate={0.0}, align={center}}, legend style={color={rgb,1:red,0.0;green,0.0;blue,0.0}, draw opacity={1.0}, line width={1}, solid, fill={rgb,1:red,1.0;green,1.0;blue,1.0}, fill opacity={1.0}, text opacity={1.0}, font={{\fontsize{8 pt}{10.4 pt}\selectfont}}, text={rgb,1:red,0.0;green,0.0;blue,0.0}, cells={anchor={center}}, at={(0.2, 0.4)}, 
anchor={north west}}, 
axis background/.style={fill={rgb,1:red,1.0;green,1.0;blue,1.0}, opacity={1.0}}, 
anchor={north west}, 
xshift=0, yshift=0, 
width=0.5\textwidth, height=0.35\textwidth, 
scaled x ticks={false}, xlabel={Action Branching Factor}, x tick style={color={rgb,1:red,0.0;green,0.0;blue,0.0}, opacity={1.0}}, x tick label style={color={rgb,1:red,0.0;green,0.0;blue,0.0}, opacity={1.0}, rotate={0}}, xlabel style={at={(ticklabel cs:0.5)}, anchor=near ticklabel, at={{(ticklabel cs:0.5)}}, anchor={near ticklabel}, font={{\fontsize{11 pt}{14.3 pt}\selectfont}}, color={rgb,1:red,0.0;green,0.0;blue,0.0}, draw opacity={1.0}, rotate={0.0}}, xmajorgrids={true}, xmin={3.1899999999999995}, xmax={31.810000000000002}, xticklabels={{$5$,$10$,$15$,$20$,$25$,$30$}}, xtick={{5.0,10.0,15.0,20.0,25.0,30.0}}, xtick align={inside}, xticklabel style={font={{\fontsize{8 pt}{10.4 pt}\selectfont}}, color={rgb,1:red,0.0;green,0.0;blue,0.0}, draw opacity={1.0}, rotate={0.0}}, x grid style={color={rgb,1:red,0.0;green,0.0;blue,0.0}, draw opacity={0.1}, line width={0.5}, solid}, axis x line*={left}, x axis line style={color={rgb,1:red,0.0;green,0.0;blue,0.0}, draw opacity={1.0}, line width={1}, solid}, scaled y ticks={false}, ylabel={Mean Reward ($\hat{V}_R$)}, y tick style={color={rgb,1:red,0.0;green,0.0;blue,0.0}, opacity={1.0}}, y tick label style={color={rgb,1:red,0.0;green,0.0;blue,0.0}, opacity={1.0}, rotate={0}}, ylabel style={at={(ticklabel cs:0.5)}, anchor=near ticklabel, at={{(ticklabel cs:0.5)}}, anchor={near ticklabel}, font={{\fontsize{11 pt}{14.3 pt}\selectfont}}, color={rgb,1:red,0.0;green,0.0;blue,0.0}, draw opacity={1.0}, rotate={0.0}}, ymajorgrids={true}, ymin={-34.1738581087625}, ymax={73.61597188839892}, yticklabels={{$-20$,$0$,$20$,$40$,$60$}}, ytick={{-20.0,0.0,20.0,40.0,60.0}}, ytick align={inside}, yticklabel style={font={{\fontsize{8 pt}{10.4 pt}\selectfont}}, color={rgb,1:red,0.0;green,0.0;blue,0.0}, draw opacity={1.0}, rotate={0.0}}, y grid style={color={rgb,1:red,0.0;green,0.0;blue,0.0}, draw opacity={0.1}, line width={0.5}, solid}, axis y line*={left}, y axis line style={color={rgb,1:red,0.0;green,0.0;blue,0.0}, draw opacity={1.0}, line width={1}, solid}, colorbar={false}]
    \addplot[color={rgb,1:red,0.0;green,0.6056;blue,0.9787}, name path={d5a8a5ef-b432-4b65-ab49-79601f7718fb}, draw opacity={1.0}, line width={1}, solid]
        table[row sep={\\}]
        {
            \\
            4.0  66.06831311951481  \\
            7.0  65.25781338034882  \\
            10.0  63.9906667297457  \\
            13.0  66.54574500785769  \\
            16.0  69.42498699725874  \\
            19.0  64.62557837574109  \\
            22.0  55.86997322331249  \\
            25.0  43.1069735407971  \\
            28.0  52.79376251820936  \\
            31.0  29.140127989807585  \\
        }
        ;
    \addlegendentry {\algname-Unc.}
    \addplot[color={rgb,1:red,0.0;green,0.0;blue,0.0}, name path={2c32d180-3003-482e-ba09-0d09dd2272fa}, draw opacity={1.0}, line width={1}, solid, mark={-}, mark size={3.0 pt}, mark repeat={1}, mark options={color={rgb,1:red,0.0;green,0.0;blue,0.0}, draw opacity={1.0}, fill={rgb,1:red,0.0;green,0.0;blue,0.0}, fill opacity={1.0}, line width={0.75}, rotate={0}, solid}, forget plot]
        table[row sep={\\}]
        {
            \\
            4.0  64.89764194091745  \\
            4.0  67.23898429811217  \\
        }
        ;
    \addplot[color={rgb,1:red,0.0;green,0.0;blue,0.0}, name path={2c32d180-3003-482e-ba09-0d09dd2272fa}, draw opacity={1.0}, line width={1}, solid, mark={-}, mark size={3.0 pt}, mark repeat={1}, mark options={color={rgb,1:red,0.0;green,0.0;blue,0.0}, draw opacity={1.0}, fill={rgb,1:red,0.0;green,0.0;blue,0.0}, fill opacity={1.0}, line width={0.75}, rotate={0}, solid}, forget plot]
        table[row sep={\\}]
        {
            \\
            7.0  62.48648264879825  \\
            7.0  68.0291441118994  \\
        }
        ;
    \addplot[color={rgb,1:red,0.0;green,0.0;blue,0.0}, name path={2c32d180-3003-482e-ba09-0d09dd2272fa}, draw opacity={1.0}, line width={1}, solid, mark={-}, mark size={3.0 pt}, mark repeat={1}, mark options={color={rgb,1:red,0.0;green,0.0;blue,0.0}, draw opacity={1.0}, fill={rgb,1:red,0.0;green,0.0;blue,0.0}, fill opacity={1.0}, line width={0.75}, rotate={0}, solid}, forget plot]
        table[row sep={\\}]
        {
            \\
            10.0  60.87825577773901  \\
            10.0  67.10307768175238  \\
        }
        ;
    \addplot[color={rgb,1:red,0.0;green,0.0;blue,0.0}, name path={2c32d180-3003-482e-ba09-0d09dd2272fa}, draw opacity={1.0}, line width={1}, solid, mark={-}, mark size={3.0 pt}, mark repeat={1}, mark options={color={rgb,1:red,0.0;green,0.0;blue,0.0}, draw opacity={1.0}, fill={rgb,1:red,0.0;green,0.0;blue,0.0}, fill opacity={1.0}, line width={0.75}, rotate={0}, solid}, forget plot]
        table[row sep={\\}]
        {
            \\
            13.0  63.84336581759279  \\
            13.0  69.24812419812258  \\
        }
        ;
    \addplot[color={rgb,1:red,0.0;green,0.0;blue,0.0}, name path={2c32d180-3003-482e-ba09-0d09dd2272fa}, draw opacity={1.0}, line width={1}, solid, mark={-}, mark size={3.0 pt}, mark repeat={1}, mark options={color={rgb,1:red,0.0;green,0.0;blue,0.0}, draw opacity={1.0}, fill={rgb,1:red,0.0;green,0.0;blue,0.0}, fill opacity={1.0}, line width={0.75}, rotate={0}, solid}, forget plot]
        table[row sep={\\}]
        {
            \\
            16.0  68.28465767207597  \\
            16.0  70.56531632244152  \\
        }
        ;
    \addplot[color={rgb,1:red,0.0;green,0.0;blue,0.0}, name path={2c32d180-3003-482e-ba09-0d09dd2272fa}, draw opacity={1.0}, line width={1}, solid, mark={-}, mark size={3.0 pt}, mark repeat={1}, mark options={color={rgb,1:red,0.0;green,0.0;blue,0.0}, draw opacity={1.0}, fill={rgb,1:red,0.0;green,0.0;blue,0.0}, fill opacity={1.0}, line width={0.75}, rotate={0}, solid}, forget plot]
        table[row sep={\\}]
        {
            \\
            19.0  60.22694345859317  \\
            19.0  69.024213292889  \\
        }
        ;
    \addplot[color={rgb,1:red,0.0;green,0.0;blue,0.0}, name path={2c32d180-3003-482e-ba09-0d09dd2272fa}, draw opacity={1.0}, line width={1}, solid, mark={-}, mark size={3.0 pt}, mark repeat={1}, mark options={color={rgb,1:red,0.0;green,0.0;blue,0.0}, draw opacity={1.0}, fill={rgb,1:red,0.0;green,0.0;blue,0.0}, fill opacity={1.0}, line width={0.75}, rotate={0}, solid}, forget plot]
        table[row sep={\\}]
        {
            \\
            22.0  49.13888913003837  \\
            22.0  62.60105731658661  \\
        }
        ;
    \addplot[color={rgb,1:red,0.0;green,0.0;blue,0.0}, name path={2c32d180-3003-482e-ba09-0d09dd2272fa}, draw opacity={1.0}, line width={1}, solid, mark={-}, mark size={3.0 pt}, mark repeat={1}, mark options={color={rgb,1:red,0.0;green,0.0;blue,0.0}, draw opacity={1.0}, fill={rgb,1:red,0.0;green,0.0;blue,0.0}, fill opacity={1.0}, line width={0.75}, rotate={0}, solid}, forget plot]
        table[row sep={\\}]
        {
            \\
            25.0  34.18223066596333  \\
            25.0  52.03171641563087  \\
        }
        ;
    \addplot[color={rgb,1:red,0.0;green,0.0;blue,0.0}, name path={2c32d180-3003-482e-ba09-0d09dd2272fa}, draw opacity={1.0}, line width={1}, solid, mark={-}, mark size={3.0 pt}, mark repeat={1}, mark options={color={rgb,1:red,0.0;green,0.0;blue,0.0}, draw opacity={1.0}, fill={rgb,1:red,0.0;green,0.0;blue,0.0}, fill opacity={1.0}, line width={0.75}, rotate={0}, solid}, forget plot]
        table[row sep={\\}]
        {
            \\
            28.0  45.397889900466055  \\
            28.0  60.18963513595267  \\
        }
        ;
    \addplot[color={rgb,1:red,0.0;green,0.0;blue,0.0}, name path={2c32d180-3003-482e-ba09-0d09dd2272fa}, draw opacity={1.0}, line width={1}, solid, mark={-}, mark size={3.0 pt}, mark repeat={1}, mark options={color={rgb,1:red,0.0;green,0.0;blue,0.0}, draw opacity={1.0}, fill={rgb,1:red,0.0;green,0.0;blue,0.0}, fill opacity={1.0}, line width={0.75}, rotate={0}, solid}, forget plot]
        table[row sep={\\}]
        {
            \\
            31.0  19.246927001746947  \\
            31.0  39.03332897786822  \\
        }
        ;
    \addplot[color={rgb,1:red,0.8889;green,0.4356;blue,0.2781}, name path={81589b09-deb5-48a0-82d8-dca70c849e20}, draw opacity={1.0}, line width={1}, solid]
        table[row sep={\\}]
        {
            \\
            4.0  66.06831311951481  \\
            7.0  65.25781338034882  \\
            11.0  61.935431124633645  \\
            15.0  62.12282207124268  \\
            19.0  24.767524293484733  \\
            23.0  -2.447285530242524  \\
            27.0  -21.836569012267674  \\
            31.0  -4.6196359575244665  \\
        }
        ;
    \addlegendentry {\algname-Random}
    \addplot[color={rgb,1:red,0.0;green,0.0;blue,0.0}, name path={ddd7a762-4c14-4e8a-9489-954fa0fcfa36}, draw opacity={1.0}, line width={1}, solid, mark={-}, mark size={3.0 pt}, mark repeat={1}, mark options={color={rgb,1:red,0.0;green,0.0;blue,0.0}, draw opacity={1.0}, fill={rgb,1:red,0.0;green,0.0;blue,0.0}, fill opacity={1.0}, line width={0.75}, rotate={0}, solid}, forget plot]
        table[row sep={\\}]
        {
            \\
            4.0  64.89764194091745  \\
            4.0  67.23898429811217  \\
        }
        ;
    \addplot[color={rgb,1:red,0.0;green,0.0;blue,0.0}, name path={ddd7a762-4c14-4e8a-9489-954fa0fcfa36}, draw opacity={1.0}, line width={1}, solid, mark={-}, mark size={3.0 pt}, mark repeat={1}, mark options={color={rgb,1:red,0.0;green,0.0;blue,0.0}, draw opacity={1.0}, fill={rgb,1:red,0.0;green,0.0;blue,0.0}, fill opacity={1.0}, line width={0.75}, rotate={0}, solid}, forget plot]
        table[row sep={\\}]
        {
            \\
            7.0  62.48648264879825  \\
            7.0  68.0291441118994  \\
        }
        ;
    \addplot[color={rgb,1:red,0.0;green,0.0;blue,0.0}, name path={ddd7a762-4c14-4e8a-9489-954fa0fcfa36}, draw opacity={1.0}, line width={1}, solid, mark={-}, mark size={3.0 pt}, mark repeat={1}, mark options={color={rgb,1:red,0.0;green,0.0;blue,0.0}, draw opacity={1.0}, fill={rgb,1:red,0.0;green,0.0;blue,0.0}, fill opacity={1.0}, line width={0.75}, rotate={0}, solid}, forget plot]
        table[row sep={\\}]
        {
            \\
            11.0  58.564574402749024  \\
            11.0  65.30628784651826  \\
        }
        ;
    \addplot[color={rgb,1:red,0.0;green,0.0;blue,0.0}, name path={ddd7a762-4c14-4e8a-9489-954fa0fcfa36}, draw opacity={1.0}, line width={1}, solid, mark={-}, mark size={3.0 pt}, mark repeat={1}, mark options={color={rgb,1:red,0.0;green,0.0;blue,0.0}, draw opacity={1.0}, fill={rgb,1:red,0.0;green,0.0;blue,0.0}, fill opacity={1.0}, line width={0.75}, rotate={0}, solid}, forget plot]
        table[row sep={\\}]
        {
            \\
            15.0  58.934054238689164  \\
            15.0  65.31158990379619  \\
        }
        ;
    \addplot[color={rgb,1:red,0.0;green,0.0;blue,0.0}, name path={ddd7a762-4c14-4e8a-9489-954fa0fcfa36}, draw opacity={1.0}, line width={1}, solid, mark={-}, mark size={3.0 pt}, mark repeat={1}, mark options={color={rgb,1:red,0.0;green,0.0;blue,0.0}, draw opacity={1.0}, fill={rgb,1:red,0.0;green,0.0;blue,0.0}, fill opacity={1.0}, line width={0.75}, rotate={0}, solid}, forget plot]
        table[row sep={\\}]
        {
            \\
            19.0  16.210032547182017  \\
            19.0  33.32501603978745  \\
        }
        ;
    \addplot[color={rgb,1:red,0.0;green,0.0;blue,0.0}, name path={ddd7a762-4c14-4e8a-9489-954fa0fcfa36}, draw opacity={1.0}, line width={1}, solid, mark={-}, mark size={3.0 pt}, mark repeat={1}, mark options={color={rgb,1:red,0.0;green,0.0;blue,0.0}, draw opacity={1.0}, fill={rgb,1:red,0.0;green,0.0;blue,0.0}, fill opacity={1.0}, line width={0.75}, rotate={0}, solid}, forget plot]
        table[row sep={\\}]
        {
            \\
            23.0  -12.25381660468572  \\
            23.0  7.359245544200672  \\
        }
        ;
    \addplot[color={rgb,1:red,0.0;green,0.0;blue,0.0}, name path={ddd7a762-4c14-4e8a-9489-954fa0fcfa36}, draw opacity={1.0}, line width={1}, solid, mark={-}, mark size={3.0 pt}, mark repeat={1}, mark options={color={rgb,1:red,0.0;green,0.0;blue,0.0}, draw opacity={1.0}, fill={rgb,1:red,0.0;green,0.0;blue,0.0}, fill opacity={1.0}, line width={0.75}, rotate={0}, solid}, forget plot]
        table[row sep={\\}]
        {
            \\
            27.0  -31.123202542805107  \\
            27.0  -12.549935481730241  \\
        }
        ;
    \addplot[color={rgb,1:red,0.0;green,0.0;blue,0.0}, name path={ddd7a762-4c14-4e8a-9489-954fa0fcfa36}, draw opacity={1.0}, line width={1}, solid, mark={-}, mark size={3.0 pt}, mark repeat={1}, mark options={color={rgb,1:red,0.0;green,0.0;blue,0.0}, draw opacity={1.0}, fill={rgb,1:red,0.0;green,0.0;blue,0.0}, fill opacity={1.0}, line width={0.75}, rotate={0}, solid}, forget plot]
        table[row sep={\\}]
        {
            \\
            31.0  -14.750866686719156  \\
            31.0  5.511594771670224  \\
        }
        ;
    \addplot[color={rgb,1:red,0.0;green,0.0;blue,0.0}, name path={76b777c4-f1bc-4eaf-86fc-659eefe8b0c8}, draw opacity={1.0}, line width={1}, solid, mark={-}, mark size={3.0 pt}, mark repeat={1}, mark options={color={rgb,1:red,0.0;green,0.0;blue,0.0}, draw opacity={1.0}, fill={rgb,1:red,0.0;green,0.0;blue,0.0}, fill opacity={1.0}, line width={0.75}, rotate={0}, solid}, forget plot]
        table[row sep={\\}]
        {
            \\
            7.0  -4.901707326388459  \\
            7.0  16.752988015710834  \\
        }
        ;
    \addplot[color={rgb,1:red,0.2422;green,0.6433;blue,0.3044}, name path={d2a333e0-0548-4f24-b9f9-1c67eb140b6a}, only marks, draw opacity={1.0}, line width={0}, solid, mark={*}, mark size={3.0 pt}, mark repeat={1}, mark options={color={rgb,1:red,0.0;green,0.0;blue,0.0}, draw opacity={1.0}, fill={rgb,1:red,0.2422;green,0.6433;blue,0.3044}, fill opacity={1.0}, line width={0.75}, rotate={0}, solid}]
        table[row sep={\\}]
        {
            \\
            7.0  5.925640344661186  \\
        }
        ;
    \addlegendentry {CPFT-DPW}
    \addplot[color={rgb,1:red,0.2422;green,0.6433;blue,0.3044}, name path={e986d6ae-916f-4f8f-86f7-5fc3b1ae04dd}, draw opacity={1.0}, line width={1}, dashed, forget plot]
        table[row sep={\\}]
        {
            \\
            3.1899999999999995  5.925640344661186  \\
            31.810000000000002  5.925640344661186  \\
        }
        ;
    \addplot[color={rgb,1:red,0.0;green,0.0;blue,0.0}, name path={ac338885-64a2-4b16-b027-ef555d8a8f91}, draw opacity={1.0}, line width={1}, solid, mark={-}, mark size={3.0 pt}, mark repeat={1}, mark options={color={rgb,1:red,0.0;green,0.0;blue,0.0}, draw opacity={1.0}, fill={rgb,1:red,0.0;green,0.0;blue,0.0}, fill opacity={1.0}, line width={0.75}, rotate={0}, solid}, forget plot]
        table[row sep={\\}]
        {
            \\
            7.0  -20.54198972897756  \\
            7.0  2.077126229262543  \\
        }
        ;
    \addplot[color={rgb,1:red,0.7644;green,0.4441;blue,0.8243}, name path={17503e99-6749-4778-a947-0c9f57e60214}, only marks, draw opacity={1.0}, line width={0}, dotted, mark={*}, mark size={3.0 pt}, mark repeat={1}, mark options={color={rgb,1:red,0.0;green,0.0;blue,0.0}, draw opacity={1.0}, fill={rgb,1:red,0.7644;green,0.4441;blue,0.8243}, fill opacity={1.0}, line width={0.75}, rotate={0}, solid}]
        table[row sep={\\}]
        {
            \\
            7.0  -9.232431749857508  \\
        }
        ;
    \addlegendentry {CPOMCPOW}
    \addplot[color={rgb,1:red,0.7644;green,0.4441;blue,0.8243}, name path={e986d6ae-916f-4f8f-86f7-5fc3b1ae04dd}, draw opacity={1.0}, line width={1}, dashed]
        table[row sep={\\}]
        {
            \\
            3.1899999999999995  -9.232431749857508  \\
            31.810000000000002  -9.232431749857508  \\
        }
        ;
\end{axis}
\end{tikzpicture}

%% file: sections/05-conclusion.tex
\section{Conclusion}

Safe robotic planning under state and transition uncertainty can be naturally expressed as a CPOMDP, but CPOMDPs are extremely difficult to solve in large problem domains. Recent works scaled online search-based CPOMDP planning to large spaces~\cite{lee2018monte,jamgochian2023online}, but with limited scope or expertly crafted heuristics. In robotics, planning can often be favorably decomposed hierarchically, separating high-level action primitives and low-level control. In this work, we introduced Constrained Options Belief Tree Search (\algname) to improve online CPOMDP planning when favorable hierarchies exist by performing a belief-state Monte Carlo tree search over options. We showed that \algname\ with feasible options will satisfy constraints anytime and demonstrated its success on large planning domains where recent methods fail.

\textbf{Limitations and future work:} A significant limitation of \algname\ is the necessity to craft low-level policies. Recent work uses language models to construct policies automatically~\cite{liang2023code} and if combined with \algname, could enable hierarchical constrained search to compensate for uncertainty in policy generation. A second limitation is that though \algname\ biases the search toward safety, it requires feasible options in order to satisfy constraints anytime. Future work can examine propagating cost bounds \cite{undurti2010online} or using search heuristics generated from past experience~\cite{cai2022closing,moss2023betazero,parthasarathy2023cmcts} or natural language priors~\cite{zhao2023large} to achieve safety under more general conditions. 